\documentclass[10pt,journal,compsoc]{IEEEtran}

\usepackage{algorithm, algorithmicx, algpseudocode}
\usepackage{hyperref}       
\usepackage{url}            
\usepackage{booktabs}       
\usepackage{nicefrac}       
\usepackage{microtype}      
\usepackage{xcolor}         
\usepackage{caption}
\usepackage{amsmath, amssymb, amsfonts, amsthm, mathtools, bbold, relsize}
\usepackage{siunitx}
\usepackage{graphicx}	
\usepackage{textcomp}

\newtheorem{theorem}{Theorem}

\newtheorem{corollary}{Corollary}
\newcommand{\xsup}[1]{{x^{(#1)}}}
\newcommand{\ysup}[1]{{y^{(#1)}}}

\newcommand{\comment}[1]{}

\newcommand{\tr}{\text{tr}\,}
\newcommand{\var}{\text{var}\,}
\newcommand{\cov}{\text{cov}\,}

\newcommand{\norm}[1]{\left\lVert #1 \right\rVert}

\newcommand{\mindN}{\underline{d,N}}
\newcommand{\estoosmse}[2]{\widehat{\text{mse}}_{\text{oos}}(#1,#2)}
\newcommand{\figurewidth}{0.9\linewidth}
\makeatletter
\renewcommand{\maketag@@@}[1]{\hbox{\m@th\normalsize\normalfont#1}}%
\makeatother

\ifCLASSOPTIONcompsoc
  \usepackage[nocompress]{cite}
\else
  \usepackage{cite}
\fi

\begin{document}
\bstctlcite{IEEEexample:BSTcontrol}

\title{Optimal ridge regularization revisited}

\author{Jack Timmermans and Sergio A.\ Alvarez

\IEEEcompsocitemizethanks{\IEEEcompsocthanksitem The authors are with the
Department of Computer Science, Boston College, Chestnut Hill, MA 02467 USA.
J.\ Timmerman's work was supported by an undergraduate research fellowship from the Morrissey College of Arts and Sciences at Boston College. 
\ Correspondence to:\ \, alvarez@bc.edu.
}

} 

\markboth{Preprint}
{Timmermans and Alvarez\MakeLowercase{\textit{}}: Optimal ridge regularization revisited}
%



\IEEEtitleabstractindextext{%
\begin{abstract}
We consider $L^2$-regularized linear (ridge) regression over a finite data sample $X$ with bounded covariance and linear prediction targets $y$ with additive isotropic noise of finite variance. We present an iterative procedure to compute the optimal regularization strength numerically from the generative parameters in the fixed-$X$ setting and prove its convergence at limited noise levels. Our experimental evaluation over synthetic data shows that the proposed procedure combined with sample-based parameter estimates attains near-optimal random-$X$ generalization across a wide range of sample sizes, aspect ratios, and noise levels, at an added computational cost equivalent to one preliminary ridge regression in the underparameterized regime and two in the overparameterized case.
\end{abstract}

}

\maketitle

\IEEEdisplaynontitleabstractindextext

%



\IEEEraisesectionheading{\section{Introduction}}

\IEEEPARstart{R}
idge regression learns a linear model $y = \hat{\theta}^T x$ to predict a real-valued target variable $y$ from multivariate input data $x \in \mathbb R^{d \times 1}$. 
The model parameter, $\hat \theta \in \mathbb R^{d \times 1}$, is found by empirical risk minimization of the $L^2$-penalized squared loss over a labeled data sample 
(Eq.~\ref{eq:ermRidgeRegressionMatrixForm}),
where the rows of $X \in \mathbb R^{N \times d}$ and $y \in \mathbb R^{N \times 1}$ are inputs $\xsup{i}^T$ and their labels, $\ysup{i}$, and $\lambda$ is a regularization strength hyperparameter.
\begin{align}
    \hat{\theta} = \arg\min_{\eta} 
    \ (X\eta - y)^T(X\eta - y)
    + \lambda \eta^T\eta
    \label{eq:ermRidgeRegressionMatrixForm}
\end{align}
As described in Appendix B of~\cite{patil2024optimal} (note the factor $N$ difference in scaling of $\lambda$ in that paper as compared with~\ref{eq:ermRidgeRegressionMatrixForm}), the objective function of Eq.~\ref{eq:ermRidgeRegressionMatrixForm}
is strictly convex if $\lambda > -\sigma^2_{\text{min}}$ (even if $\lambda < 0$), where $\sigma_{\text{min}}$ is the smallest singular value of $X$. Hence, there is a unique minimum point, which can be found by the first-order stationarity condition of
Eq.~\ref{eq:ridgeParameterExpression} (e.g.,~\cite{hastie_09_elements-of-statistical-learning}), where $A^{\dagger}$
is the Moore-Penrose pseudo-inverse (which reduces to the ordinary inverse in the underparameterized case $d \le N$,
and to the inverse of the Gram matrix $XX^T + \lambda I$ in the overparameterized case $d > N$).
\begin{align}
\hat{\theta} = (X^T X + \lambda I)^{\dagger}X^Ty
\label{eq:ridgeParameterExpression}
\end{align}
Section 3b of~\cite{hoerl1970ridge} notes that the parameter estimate of Eq.~\ref{eq:ridgeParameterExpression} can be characterized as minimizing $L^2$ norm $\norm{\theta}$ subject to constant penalized squared error $\norm{X\theta - y}^2 + \lambda \Vert \theta \Vert^2$, and as minimizing this error subject to constant $L^2$ norm $\Vert \theta \Vert$. 
Eq.~\ref{eq:ridgeParameterExpression} can also be interpreted within a Bayesian framework, as the maximum a posteriori
(MAP) estimate of $\theta$ based on combining a Gaussian conditional likelihood $N(X\theta, I)$ for $y$ with an isotropic Gaussian prior $N(0, \frac{1}{\lambda} I)$ for $\theta$~\cite{DeepLearningBook2016}; thus,
larger ridge penalty strengths $\lambda$ correspond to a narrower prior that biases $\theta$
to be smaller in norm.

Ridge regression was introduced in~\cite{hoerl1962ridge}; similar ideas were described earlier for integral equations~\cite{phillips1962} and more general ill-posed problems~(see~\cite{Tikhonov:1963, tikhonov1977});
the technique is sometimes referred to as Tikhonov regularization in connection with the latter. Additional properties were considered in~\cite{hoerl1970ridge}. The technique remains a staple of statistical machine learning~\cite{hastie_09_elements-of-statistical-learning}, and has been the subject of many publications in the field. Recent theoretical work emphasizes large-sample asymptotics at constant aspect ratio $d/N$ in the overparameterized case, where optimal regularization strength may be negative~\cite{DobribanWagnerHighDAsymptotics2018,kobak2020optimal,WuXuNeurIPS2020,tsigler2023benign,patil2024optimal}.
$L^2$ regularization, more generally, is widely used in machine learning as a mechanism for controlling model capacity;
examples occur in kernel learning~\cite{CortesL2KernelLearning2009,gonen2011multiple} and neural
networks~\cite{L2OverparameterizedNN2021,NEURIPS2020_32fcc8cf}, among many others.
\cite{srivastava2014dropout} showed that dropout can be interpreted as $L^2$ regularization using a
covariance-sensitive norm.

The regularization strength hyperparameter (ridge penalty) $\lambda$ in Eqs.~\ref{eq:ermRidgeRegressionMatrixForm}- \ref{eq:ridgeParameterExpression} is typically selected by a search process over a validation set (e.g.,~\cite{buteneers2013optimized}), requiring multiple ridge regression fits to compare 
generalization error estimates. Several ``plug-in'' choices for $\lambda$ have also been proposed, some dating
back to~\cite{HoerlPlugIn1975}. In practice, these techniques require an initial model fit, 
and typically rely on heuristic estimates rather than exact
descriptions of the optimal regularization strength. Assuming an isotropic random linear generative parameter $\theta$,~\cite{DobribanWagnerHighDAsymptotics2018} describes the asymptotically optimal regularization strength explicitly in terms of the variances of $\theta$ (``signal'') and of the noise, but without accounting for issues that arise for specific finite data samples, including the direction of $\theta$ relative to the covariance geometry of $X$, and sample-based estimation of the generative parameters. 

\subsection*{Scope of this paper}
We consider the problem of efficiently computing the optimal regularization strength, $\lambda$ in Eq.~\ref{eq:ridgeParameterExpression}, based on a given labeled data sample $(X,y)$; this is the value that
minimizes expected out-of-sample in-distribution squared error under the data generation model of Eq.~\ref{eq:dataGenerativeModel}. 
We begin in section~\ref{section:biasVariance} by expressing the expected out-of-sample mean squared error 
analytically in the fixed-X setting, in terms of the sample covariance of $X$ and of the data-generative population parameters for $y$, assuming these are known;
the analysis in section~\ref{section:biasVariance} reflects standard knowledge in the field.
We then derive in section~\ref{section:computingOptimalRegWeight} an equation for the optimal regularization strength and prove that a fixed-point computational procedure that uses the generative parameters (Algorithm~\ref{alg:iterative-optimal-regularization} in section~\ref{subsection:modelBasedOptimalRegularization}) converges to the optimal $\lambda$ for sufficiently small values of the noise variance. 
To our knowledge, this result has not previously been reported in the literature.

Next, in~\ref{subsection:sampleBasedParameterEstimation}, we describe an approach for sample-based estimation of the label-generative parameters that are
needed for the optimal regularization computation, requiring only one additional preliminary ridge regression
in the underparameterized case and two in the overparameterized case.
This yields our proposed sample-based optimal regularization approach (Algorithm~\ref{alg:sample-based-optimal-regularization}).
We report on our experimental evaluation in section~\ref{section:experimentalIllustrations} using controllable
synthetic spiked and lump covariance data. The results suggest
that Alg.~\ref{alg:sample-based-optimal-regularization} yields near-optimal random-$X$ generalization across a wide range of aspect ratios $d/N$ and noise levels $\epsilon$, 
considerably better than a data-agnostic default regularization strength
and better than a sample-based estimate of the signal-to-noise asymptotically optimal regularization strength of~\cite{DobribanWagnerHighDAsymptotics2018} at higher aspect ratios and noise levels.

\section{Isotropic noise model; bias-variance}
\label{section:biasVariance}

\subsection{Data-generative model}
\label{subsection:dataGenerativeModel}
We will assume the data generation mechanisms of Eq.~\ref{eq:dataGenerativeModel},
for a covariance matrix $\Sigma \in \mathbb R^{d \times d}$, parameter vector $\theta \in \mathbb R^{d \times 1}$, and noise amplitude $\epsilon \in \mathbb R^+$. The rows of $X \in \mathbb R^{N \times d}$ and $y \in \mathbb R^{N \times 1}$ are sampled (i.i.d.) from Eq.~\ref{eq:dataGenerativeModel} and mean-centered.
See Algorithm~\ref{alg:data-generation-model}.
\begin{subequations}
\begin{align}
    x &\sim p(x), &\text{ with } &E(p) = 0,\ &\cov p = \Sigma
        \label{eq:a}
\\
    y &= x^T \theta + \epsilon z, 
    &\text{ where } &E(z) = 0,\ &\var z = 1
        \label{eq:b}
\end{align}
\label{eq:dataGenerativeModel}
\end{subequations}
\noindent
The model of Eq.~\ref{eq:dataGenerativeModel} appears often in the research literature. 
Two distinct settings of this data generation model can be considered when evaluating generalization~\cite{rosset2020fixedx}: fixed-$X$, in which $y$ is subject to additive noise as in Eq.~\ref{eq:b} while $X$ remains fixed, and random-$X$, in which both $X$ and $y$ are sampled randomly as in Eq.~\ref{eq:a} and Eq.~\ref{eq:b}, respectively.

\begin{algorithm}
\begin{algorithmic}[1]
\Function{GenXData}{$n, d, \Sigma$}
\State $X \in \mathbb R^{n \times d} \overset{\text{i.i.d.}}{\sim} p(x)$ as in Eq.~\ref{eq:a}
\State \Return{$(I - \frac{1}{n}\mathbb{1 1}^T) X$\quad (mean-center columns)}
\EndFunction
\State
\Function{GenYData}{$X, \theta, \epsilon$}
\State $y = X \theta + \epsilon z$ as in Eq.~\ref{eq:b} ($z$ i.i.d.\ across rows)
\State \Return{$(I - \frac{1}{n} \mathbb{1 1}^T) y$\quad where $n = $ height of $X, y$}
\EndFunction
\end{algorithmic}
\caption{Data generation model.}
\label{alg:data-generation-model}
\end{algorithm}

\subsection{Bias-variance decomposition of prediction error}
\label{subsection:biasVarianceDecomposition}

We use Eq.~\ref{eq:ridgeParameterExpression} to compute an estimate $\hat{\theta}$ of the generative parameter $\theta$ for a data sample $(X, y)$ obtained via Algorithm~\ref{alg:data-generation-model}, and consider the variation of the predicted labels vector $\hat{y} = X\hat{\theta}$ due to the noise term $\epsilon z$ only, keeping $X$ fixed while accounting for the covariance structure in Eq.~\ref{eq:a} (fixed-$X$ setting). The result is standard and we borrow from~\cite{AvatiStanfordNotes2019,hastie_09_elements-of-statistical-learning}.

Let $X = USV^T$ be the singular value decomposition (SVD) of the unlabeled data matrix, $X$, where $U \in \mathbb R^{N \times N}$ and $V \in \mathbb R^{d \times d}$ are (self-) orthogonal: $U^T U = I_N$, $V^T V = I_d$, and $S \in \mathbb R^{N \times d}$ is diag($\sigma_1, \cdots, \sigma_{\min\!{(d,N)}}$)~\cite{golub2013matrix}.
We assume $S$ is nonsingular (i.e., all $\sigma_j > 0$).
We denote by $v_j$ the $j$th right singular vector of $X$, that is, the $j$th column of $V$.

By orthogonality of $U$, $V$, Eq.~\ref{eq:ridgeParameterExpression} takes the form of Eq.~\ref{eq:ridgeParameterSVD}.
\begin{align}
    \hat{\theta} &= 
V \text{diag}
\left (
\frac{\sigma_j}{\sigma^2_j + \lambda}
\right )
U^T y \nonumber\\
&=
V \text{diag}
\left (
\frac{\sigma^2_j}{\sigma^2_j + \lambda}
\right )
V^T \theta
+
\epsilon
V \text{diag}
\left (
\frac{\sigma_j}{\sigma^2_j + \lambda}
\right )
U^T z
\label{eq:ridgeParameterSVD}
\end{align}
Eq.~\ref{eq:ridgeParameterSVD} and the SVD $X = USV^T$ yield the expression for the vector of predicted labels $\hat{y} = X\hat{\theta}$ in Eq.~\ref{eq:ridgePredictionSVD}.
\begin{align}
\hat{y} &= 
U \text{diag}
\left (
\frac{\sigma^3_j}{\sigma^2_j + \lambda}
\right )
V^T \theta
+
\epsilon
U \text{diag}
\left (
\frac{\sigma^2_j}{\sigma^2_j + \lambda}
\right )
U^T z
\label{eq:ridgePredictionSVD}    
\end{align}
The predictive bias and variance~\cite{hastie_09_elements-of-statistical-learning} of $\hat{y}$ follow easily from Eq.~\ref{eq:ridgePredictionSVD},
as shown in Eqs.~\ref{eq:ridgePredictionBias}-\ref{eq:ridgePredictionVar}, where $\mindN$ denotes $\min\!{(d,N)}$;
see appendix~\ref{subsection:biasVarianceSupplement}.
Note that the predictive variance in Eq.~\ref{eq:ridgePredictionVar} is the
usual (statistical) variance of the difference $\hat{y} - E_z \hat{y}$.
\begin{align}
\text{bias}^2(\hat{y}) 
= (E_z \hat{y} - y)^2
&= 
\sum_{j=1}^{\mindN} 
\left ( 
\frac{\lambda \sigma_j v^T_j \theta}{\sigma^2_j + \lambda}
\right )^{\!2}
\label{eq:ridgePredictionBias}
\\
\var(\hat{y})
=
E_z [(\hat{y} - E_z \hat{y})^2]
&=
\epsilon^2
\sum_{j=1}^{\mindN}
\left (
\frac{\sigma^2_j}{\sigma^2_j + \lambda}
\right )^{\!2}
\label{eq:ridgePredictionVar}
\end{align}
The bias-variance decomposition~\cite{hastie_09_elements-of-statistical-learning} expresses the expected out-of-sample mean squared error as the sum of bias, variance, and noise as in
Eq.~\ref{eq:mseFromBiasVarianceNoise}. 
\begin{align}
\label{eq:mseFromBiasVarianceNoise}
E \left (
\frac{1}{N} \sum_{i=1}^N 
(\hat{y}^{(i)} - \ysup{i})^2
\right )
=
\frac{1}{N}
\left ( 
\text{bias}^2(\hat{y}) + \var(\hat{y}) \right )
+ \epsilon^2
\end{align}

\section{The optimal regularization strength}
\label{section:computingOptimalRegWeight}

In light of Eqs.~\ref{eq:ridgePredictionBias}-\ref{eq:mseFromBiasVarianceNoise}, adjusting regularization strength, $\lambda$, to minimize expected out-of-sample mean squared error requires first-order stationarity with respect to $\lambda$ of the sum of bias and variance. Direct differentiation yields Eq.~\ref{eq:stationarityOptimalLambda},
where $\mindN$ denotes $\min(d, N)$. Also see appendix~\ref{subsection:mse-stationarity-appendix}.
\begin{align}
\lambda \sum_{j=1}^{\mindN} 
\frac{\sigma_j^4\, (v^T_j\theta)^2}
{(\sigma_j^2 + \lambda)^3}
 = \epsilon^2\sum_{j=1}^{\mindN} \frac{\sigma_j^4}{(\sigma_j^2+\lambda)^3}
 \label{eq:stationarityOptimalLambda}
\end{align}
We assume $\theta \ne 0$ and $\sigma_j > 0$
(so at least one projection $v^T_j \theta \ne 0$) and isolate $\lambda$ on the left of Eq.~\ref{eq:stationarityOptimalLambda},
yielding 
Eq.~\ref{eq:lambdaUpdateOperator}.
{\large
\begin{align}
\lambda 
 = 
\epsilon^2 H(\lambda),
\quad
H(\lambda)
=
\frac
{\sum_{j=1}^{\mindN} \frac{\sigma_j^4}{(\sigma_j^2+\lambda)^3}}
 {\sum_{j=1}^{\mindN} 
\frac{\sigma_j^4\, (v^T_j\theta)^2}
{(\sigma_j^2 + \lambda)^3}}
 \label{eq:lambdaUpdateOperator}
\end{align}
}

We now show that the properties of $H(\lambda)$ in Eq.~\ref{eq:lambdaUpdateOperator} enable efficient computation of the optimal regularization strength. Theorem~\ref{thm:contractivenessOfLambdaOperator} constitutes the main observation.

\begin{theorem}
In Eq.~\ref{eq:lambdaUpdateOperator},
$\lambda \mapsto \epsilon^2 H(\lambda)$ is a contractive mapping on $\mathbb R^{\ge 0}$ when $\epsilon$ is sufficiently small: there exist $\epsilon_0 \in \mathbb R^+$ and $\mu < 1$ so that $\epsilon^2 |H(\lambda) - H(\lambda')| \le \mu |\lambda-\lambda'|$ for all $\lambda, \lambda' \ge 0$ if $\epsilon \le \epsilon_0$.
\label{thm:contractivenessOfLambdaOperator}
\end{theorem}
\begin{proof}
    It suffices to show $H$ is Lipschitz-continuous on $\mathbb R^{\ge 0}$.
    Calculation shows that $\frac{dH(\lambda)}{d\lambda} = 3(A(\lambda) - B(\lambda))$, with
    \begin{align}
    A(\lambda) &=
    \frac{\sum_{j=1}^{\mindN}\frac{\sigma_j^4}{(\sigma_j^2 + \lambda)^3}\ 
    \sum_{j=1}^{\mindN} \frac{\sigma_j^4}{(\sigma_j^2 + \lambda)^4} 
    (v^T_j \theta)^2}
    { \left (
    \sum_{j=1}^{\mindN}\frac{\sigma_j^4}{(\sigma_j^2 + \lambda)^3} (v^T_j \theta)^2
    \right )^2} \nonumber \\ 
    B(\lambda) &= 
    \frac{\sum_{j=1}^{\mindN}\frac{\sigma_j^4}{(\sigma_j^2 + \lambda)^4}}
    {\sum_{j=1}^{\mindN} \frac{\sigma_j^4}{(\sigma_j^2 + \lambda)^3} 
    (v^T_j \theta)^2}
    \label{eq:AandBtermsInHPrime}
    \end{align}
    The expressions $A(\lambda)$ and $B(\lambda)$ in Eq.~\ref{eq:AandBtermsInHPrime} approach finite limits as $\lambda \rightarrow 0$,
    since all $\sigma^2_j$ are positive and at least one projection $v^T_j \theta \ne 0$. $A(\lambda)$, $B(\lambda)$ approach $0$ as $\lambda \rightarrow \infty$, as seen by writing each term $\sigma^2_j + \lambda$ as $\lambda(\frac{\sigma^2_j}{\lambda} + 1)$ and factoring out the leading $\lambda$ in numerators and denominators. Hence, $H'(\lambda)$ is continuous and bounded on $[0, \infty)$, and therefore $H$ has a finite Lipschitz constant. This completes the proof. 
\end{proof}

\subsection{Model-based optimal ridge penalty computation}
\label{subsection:modelBasedOptimalRegularization}

\begin{corollary}
There is some $\epsilon_0 > 0$ such that Alg.~\ref{alg:iterative-optimal-regularization} converges to the optimal regularization weight $\lambda^*$ (as in Eq.~\ref{eq:stationarityOptimalLambda}), for all $\epsilon \le \epsilon_0$.
\begin{proof}
Corollary~\ref{corollary:lambdaFixedPointConvergence} is a straightforward consequence of Theorem~\ref{thm:contractivenessOfLambdaOperator} and the Banach fixed point theorem. 
\end{proof}
\label{corollary:lambdaFixedPointConvergence}
\end{corollary}

We write $H(\lambda \mid S, V, \theta)$ for $H(\lambda)$ to note the additional objects needed to evaluate the fixed-point operator in Eq.~\ref{eq:lambdaUpdateOperator}.
Given that each evaluation of $H$ takes time $O(d^2)$ for computation of the inner products $v^T_j \theta$ (they can be computed before the main loop), and the geometric convergence rate of fixed-point iteration,
Algorithm~\ref{alg:iterative-optimal-regularization} runs in time $O(d^2 |\log \delta|)$.
\begin{algorithm}
\begin{algorithmic}[1]
\Function{ModelOptReg}{$S, V, \theta, \epsilon, \lambda_0, \delta$}
\State $\lambda = \lambda_0, \quad \lambda_p = \lambda_0 + 2\delta$
\While{$|\lambda - \lambda_p| > \delta$}
\State $\lambda_p = \lambda$
\State $\lambda = \epsilon^2 H(\lambda \mid S, V, \theta)$ \ as in Eq.~\ref{eq:lambdaUpdateOperator}
\EndWhile
\State \Return{$\lambda$, mse$(\lambda)$ using Eq.~\ref{eq:mseFromBiasVarianceNoise}}
\EndFunction
\end{algorithmic}
\caption{Model-based optimal ridge regularization.}
\label{alg:iterative-optimal-regularization}
\end{algorithm}

\subsection{Sample-based estimation of hidden parameters}
\label{subsection:sampleBasedParameterEstimation}
Algorithm~\ref{alg:iterative-optimal-regularization} uses the 
singular values $S$ and right singular vectors $V$,
both of which are extracted easily from the data sample, $X$.
In contrast, the generative parameter $\theta$ 
and noise amplitude $\epsilon$ (Eq.~\ref{eq:dataGenerativeModel}), which are also used,
cannot be determined directly from the data. We therefore use a
ridge estimate $\hat \theta$ (Eq.~\ref{eq:ermRidgeRegressionMatrixForm}) for $\theta$ and the estimate $\hat \epsilon$
as described below. We show experimentally in section~\ref{section:experimentalIllustrations} that gains persist despite these substitutions, even in the random $X$ setting.

\subsubsection{Underparameterized case}

We derive a sample-based estimate $\hat \epsilon$ by noting that the predicted labels vector
$\hat{y} = X \hat{\theta}$ 
includes the orthogonal projection of the isotropic noise variable $\epsilon z$ of Eq.~\ref{eq:b} onto the column space of $X$, as in the rightmost term in Eq.~\ref{eq:ridgePredictionSVD}. 
With no regularization (or ignoring its effects), this projection accounts for a fraction $\mindN/N$ of the noise variance, $\epsilon^2$;
since bias is absent in this unregularized scenario, the mean residual $\Vert \hat{y} - y \Vert^2$ is the
remaining noise variance, $(1 - \mindN/N)\epsilon^2$. If $d < N$, this yields the noise variance estimate in Eq.~\ref{eq:epsilonNoRegEstimate},
where the sample variance has been used.
\begin{align}
\hat{\epsilon}_0 &= \sqrt{ 
\frac{
\frac{1}{N-1} 
\sum_{i=1}^N \left ( \hat{y}^{(i)} - \ysup{i} \right )^2
}
{1 - \frac{d}{N}}} 
\label{eq:epsilonNoRegEstimate}
\end{align}

With regularization, Eq.~\ref{eq:ridgePredictionVar} suggests that the variance of the projection of the noise variance
onto the column space of $X$ will be reduced. However, bias will also enter as per Eq.~\ref{eq:ridgePredictionBias},
making the sample-based estimation of the noise variance less straightforward than in the unregularized scenario.
We have found that an acceptable approximation can be derived by defining a {\em regularized rank}, ${r_p}$, 
for an exponent $p \ge 0$, as in Eq.~\ref{eq:regRankExpression};
the value $p=0$ yields the usual unregularized rank: $r_0 = d$.
The regularized rank ${r_p}$ is the effective dimension of the projection onto the column space, in that a fraction $r_p/N$ of the total noise variance, $\epsilon^2$ is captured.
\begin{align}
{r_p}
&=
\sum_{j=1}^{\mindN} \left ( \frac{\sigma^2_j}{\sigma^2_j + \lambda} \right )^p
\label{eq:regRankExpression}
\end{align}

The remaining fraction $1 - {r_p}/N$ of the noise variance
enters into the residual, as in the unregularized scenario.
If we ignore the bias contribution to the residual,
we obtain the noise variance estimate
in Eq.~\ref{eq:epsilonRegEstimate}, which is well-defined whenever $\lambda > 0$ (since $r_p < N$ in that case). Experimentally, the exponent $p = \frac{1}{4}$ produces good
estimates $\hat{\epsilon}_p$ for moderate true generative noise levels $\epsilon$
(see Tables~\ref{table:med-eps-CI-random-X},~\ref{table:genParameterEstimatesCITableAppendix}).
\begin{align}
\hat{\epsilon}_p &= \sqrt{ 
\frac{
\frac{1}{N-1} 
\sum_{i=1}^N \left ( \hat{y}^{(i)} - \ysup{i} \right )^2
}
{1 - \frac{{r_p}}{N}}} 
\label{eq:epsilonRegEstimate}
\end{align}

\subsubsection{Overparameterized case}

If $d \ge N$, the unregularized model fits the training data perfectly, precluding the use of
the above approach to extract the noise from the in-sample error.
Instead, ignoring the bias term in Eq.~\ref{eq:mseFromBiasVarianceNoise} as in the underparameterized case,
we see using Eq.~\ref{eq:ridgePredictionVar} that the out-of-sample error is $\epsilon^2\, \frac{r_p}{N} + \epsilon^2$. We compute a simple 2-fold cross-validation estimate $\estoosmse{X}{y}$ of the out-of-sample error and obtain the 
estimate in Eq.~\ref{eq:epsilonRegEstimateOverparameterized}; $k > 2$ folds would reduce variability
at higher computational cost.
\begin{align}
\hat{\epsilon}_p &= \sqrt{ 
\frac{
\frac{N}{N-1} 
\estoosmse{X}{y}
}
{1 + \frac{{r_p}}{N}}} 
\label{eq:epsilonRegEstimateOverparameterized}
\end{align}

\subsection{Sample-based optimal regularization procedure}
\label{subsection:sampleBasedOptimalRegularization}
Algorithm~\ref{alg:sample-based-optimal-regularization} describes our proposed approach for sample-based optimal ridge regularization by combining an SVD-based initial ridge regression estimate $\hat{\theta}$ with the $\hat{\epsilon}$ estimation technique of
Eqs.~\ref{eq:regRankExpression}-\ref{eq:epsilonRegEstimate} and the model-based optimal ridge computation of Algorithm~\ref{alg:iterative-optimal-regularization}.
Ridge parameter estimation follows Eq.~\ref{eq:ridgeParameterExpression}, using the SVD of $X$; $(S + \lambda_0 I)^{\dagger}$ denotes the diagonal $d \times N$ matrix of the reciprocals $1/(\sigma_j + \lambda_0)$ of the shifted singular values.
Runtime is dominated by the $O(N d^2)$ SVD runtime~\cite{golub2013matrix},
assuming fixed-precision arithmetic; two (half-time) SVD
computations on size $(N/2, d)$ samples occur on line~\ref{algStep:epsilonEstimation} in the overparameterized case.
Other parameter estimation methods can be substituted in steps~\ref{algStep:thetaEstimation},~\ref{algStep:epsilonEstimation}.

\begin{algorithm}
\begin{algorithmic}[1]
\Function{SampleOptReg}{$X, y, \lambda_0, p, \delta$}
\State $U, S, V = \Call{SVD}{X}$
\label{algStep:initialSVD}
\State $\hat{\theta} = V (S + \lambda_0 I)^{\dagger} U^T y$
\label{algStep:thetaEstimation}
\State $\hat{\epsilon} = \Call{EpsilonEstimate}{S, \hat{\theta}, p}$ \quad (\text{Eq.}~\ref{eq:epsilonRegEstimate},~\ref{eq:epsilonRegEstimateOverparameterized})
\label{algStep:epsilonEstimation}
\State \Return{\Call{ModelOptReg}{$S, V, \hat{\theta}, \hat{\epsilon}, \lambda_0, \delta}$}
\EndFunction
\end{algorithmic}
\caption{Sample-based optimal ridge regularization.}
\label{alg:sample-based-optimal-regularization}
\end{algorithm}

\section{Experimental evaluation}
\label{section:experimentalIllustrations}

Below, we first describe in section~\ref{subsection:modelBasedFixedPointOptimality} experimental validation of the analysis of section~\ref{section:computingOptimalRegWeight}, which shows that model-based Algorithm~\ref{alg:iterative-optimal-regularization}, based on the true generative parameters $\theta, \epsilon$, yields optimal expected out-of-sample performance in the fixed-$X$ setting. We then move on to our experimental assessment of 
the sample-based approach of Algorithm~\ref{alg:sample-based-optimal-regularization}, starting with a description
of experimental protocol in section~\ref{subsection:experimental-protocol-sample-based}, followed by a discussion
of the results in the fixed-$X$ and random-$X$ settings (sections~\ref{section:fixed-X-results} and~\ref{section:random-X-results}, respectively).

\subsection{Optimality of model-based fixed-point Algorithm~\ref{alg:iterative-optimal-regularization}}
\label{subsection:modelBasedFixedPointOptimality}

As a preliminary step, we verified that the regularization strength returned by the fixed-point
approach in Algorithm~\ref{alg:iterative-optimal-regularization} accurately reflects the stationarity condition 
resulting from the analytical description in Eqs.~\ref{eq:ridgePredictionBias}-\ref{eq:stationarityOptimalLambda} (assuming known generative parameters in Eq.~\ref{eq:dataGenerativeModel}, 
in particular $\theta$ and $\epsilon$).

\subsubsection{Algorithm~\ref{alg:iterative-optimal-regularization} correctly identifies MSE critical points}
Consistent with Theorem~\ref{thm:contractivenessOfLambdaOperator}, Algorithm~\ref{alg:iterative-optimal-regularization} reliably returns $\lambda$ values that coincide precisely with critical points as in Eq.~\ref{eq:stationarityOptimalLambda} of the expected out-of-sample mean squared error (MSE) of Eq.~\ref{eq:mseFromBiasVarianceNoise}. 
Fig.~\ref{fig:analyticalCorrectnessFixpointExample}
shows a typical result that also illustrates the fact that the optimal value of the regularization penalty 
depends on the alignment between the generative parameter vector $\theta$ and the eigenvectors $v_j$ of the
covariance matrix $\Sigma$ (Eq.~\ref{eq:stationarityOptimalLambda}).

\begin{figure}[h!]
    \centering
    \includegraphics[width=\figurewidth, trim = 0 4 4 2, clip=True]{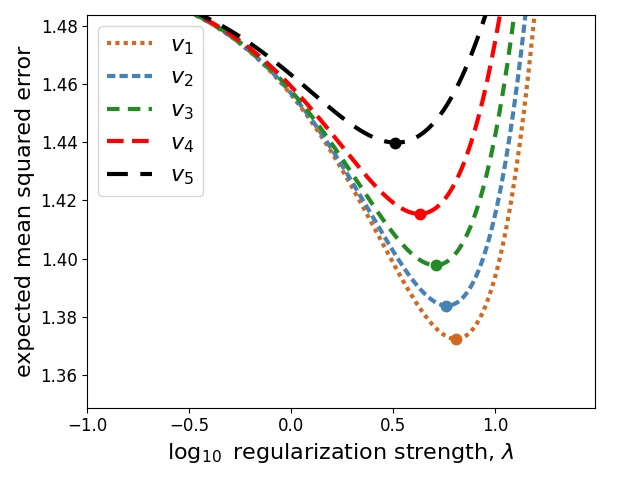}\\
    \caption{Optimality of ridge penalty $\lambda$ returned by Algorithm~\ref{alg:iterative-optimal-regularization}.
    Sample size $(N, d) = (10, 5)$. Generative covariance matrix $\Sigma = \text{diag}(5,4,3,2,1)$. 
    Noise variance $\epsilon^2=1$. $j$th curve shows exact expected out-of-sample mean squared error
    (Eq.~\ref{eq:mseFromBiasVarianceNoise}) for $j$th principal vector of $\Sigma$ as generative parameter 
    vector $\theta$ (Eq.~\ref{eq:dataGenerativeModel}); 
    each highlighted point uses the ridge penalty returned by Alg.~\ref{alg:iterative-optimal-regularization}.
    Relative differences of $\lambda$ values obtained by numerical minimization of 
    Eq.~\ref{eq:mseFromBiasVarianceNoise}
    and those returned by Alg.~\ref{alg:iterative-optimal-regularization} are less than $10^{-4}$ for $N \le 500$
    (Table~\ref{table:maxRelDiffModelBased}).}
    \label{fig:analyticalCorrectnessFixpointExample}
\end{figure}

\subsubsection{Non-global local optima exist but have minimal impact}
The first-order optimality condition of Eq.~\ref{eq:stationarityOptimalLambda} occasionally has more than one solution (e.g.,~Fig.~\ref{fig:doubleDescentExample}). 
We assessed the prevalence and potential impact of non-global local optima by comparing the values $\lambda_{\text{fp}}$ returned by Alg.~\ref{alg:iterative-optimal-regularization} with the globally MSE-minimizing $\lambda_{\text{min}}$ found by a search (Eq.~\ref{eq:lambdaStar}), for sample sizes $N = 20, 100, 500$ and $10$ aspect ratios $d/N$ over an evenly spaced range between $1/5$ and $2$ for each $N$.
\begin{align}
\lambda_{\text{min}} = \arg\min_{-1 < \lambda < 10^6} \text{mse}(\lambda)\ \text{from Eq.~\ref{eq:mseFromBiasVarianceNoise}}
\label{eq:lambdaStar}
\end{align}
Results were aggregated by taking medians over $10$ random data samples $X$ and $100$ generative vectors $\theta$ for each $(N,d)$; thus, a total of $1000$ values entered into each aggregate median for a given $d$. We gauged the degree of discrepancy by the {\em maximum} relative absolute difference 
of the $10$ pairs of median values (Table~\ref{table:maxRelDiffModelBased}).

While $\lambda$ relative differences can approach $1$ in isolated cases, MSE relative differences do not surpass $10^{-4}$, even if maxima are taken across all $10 \times 1000$ measurements for a given $N$ (values not shown).
The evidence suggests that the impact of non-global local minima on optimality of the regularization strengths returned by Alg.~\ref{alg:iterative-optimal-regularization} 
is insignificant.

\begin{table}
\caption{Maximum value of MSE relative difference $|\text{MSE}(\lambda_{\text{fp}}) - \text{MSE}(\lambda_{\text{min}})|/\text{MSE}(\lambda_{\text{min}})$ between $\lambda_{\text{fp}}$ from Alg.~\ref{alg:iterative-optimal-regularization} and globally-minimizing $\lambda_{\text{min}}$. Values are maxima over $10$ median pairs,
each an aggregate over $10$ random $X$ and $100$ generative parameters $\theta$ for a given $(N,d)$ (see text).}
\vspace{-.2cm}
\begin{center}
\renewcommand{\arraystretch}{1.2}
\begin{tabular}{cccc}
\hline
{\normalsize $N$} &{\normalsize $20$} &{\normalsize $100$} &{\normalsize $500$}\\
\hline
MSE &$8.6\, 10^{-12}$ &$1.3\, 10^{-11}$ &$1.3\, 10^{-10}$\\
$\lambda$ &$4.2\, 10^{-6}$ &$9.7\, 10^{-6}$ &$3.7\, 10^{-5}$\\
\hline
\end{tabular}
\end{center}
\label{table:maxRelDiffModelBased}
\end{table}

\begin{figure}[h!]
    \centering
    \includegraphics[width=\figurewidth, trim = 0 4 4 2, clip=True]{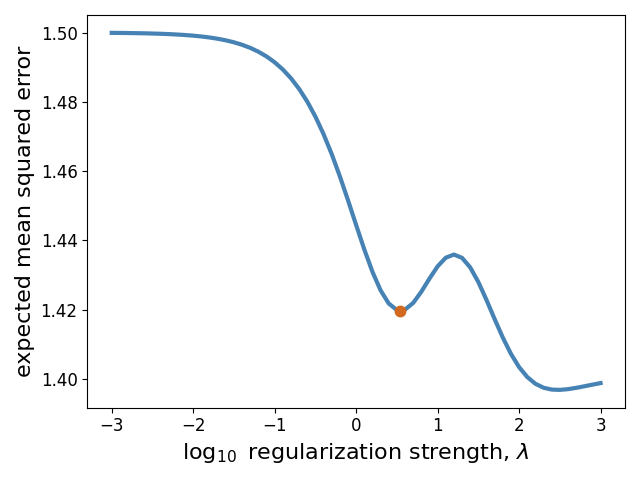}\\
    \caption{Example in which the analytical fixed-X out-of-sample expected MSE of Eq.~\ref{eq:mseFromBiasVarianceNoise} has multiple critical points (Eq.~\ref{eq:stationarityOptimalLambda}). In the example shown, Alg.~\ref{alg:iterative-optimal-regularization} returns the highlighted non-global optimum. 
    Although regularization strengths returned by Alg.~\ref{alg:iterative-optimal-regularization} occasionally 
    differ noticeably from the global minimizer, the impact on MSE appears to be insignificant
    in an aggregate sense (Table~\ref{table:maxRelDiffModelBased}).}
    \label{fig:doubleDescentExample}
\end{figure}

\subsection{Evaluation procedures in the sample-based case}
\label{subsection:experimental-protocol-sample-based}

We evaluated Alg.~\ref{alg:sample-based-optimal-regularization} in the
fixed-X and random-X settings of section~\ref{subsection:dataGenerativeModel}, with respective
evaluation protocols as in Algs.~\ref{alg:experimental-procedure-fixed-X} and~\ref{alg:experimental-procedure-random-X}; expected out-of-sample error rate is computed analytically in the fixed-X case, and estimated
over test data in the random-X case. 
Use of synthetic data allows us to control covariance profiles and label-generative parameters.

\subsubsection{Generative covariance profiles}
We considered two distinct spectral profiles for the generative covariance matrix $\Sigma$ of Eq.~\ref{eq:a},
with aspect ratios $0 < \frac{d}{N} < 2$: 
\paragraph*{\em Bulk type} eigenvalues $\sigma^2_j$
follow the arithmetic progression $1, \frac{d-1}{d}, \frac{d-2}{d}, \cdots, \frac{1}{d}$
\paragraph*{\em Spiked type} eigenvalues span three standard deviations of a Gaussian PDF. 
Eigenvalues were normalized so that the largest is $1$.

\subsubsection{Generative linear parameters}
We considered $m_{\theta}=100$ generative unit vectors $\theta$
(Eq.~\ref{eq:b}), by taking the columns of
$\lfloor \frac{100}{d} \rfloor$ orthogonal $d \times d$ matrices sampled uniformly at random, plus the 
first $100 - d \lfloor \frac{100}{d} \rfloor$ columns of another uniformly random orthogonal matrix
if $d$ is not a multiple of $100$. The value $m_{\theta}=100$ was found to be sufficient to limit
variation between runs, while not requiring excessive computation time; see~\ref{subsection:bootstrapConvergenceAppendix}.

\subsubsection{Regularization approaches considered}
\label{subsection:regularization-approaches}
We compared the following techniques for computing the ridge regularization strength:
\paragraph*{\em Optimal regularization as determined by numerical optimization}
We used the value $\lambda_{\text{min}}$ returned by the \verb|minimize| function in \verb|scipy.optimize| for the mean squared prediction error over the test sample. This value may overfit the test sample, and occasionally yields test error rates below the $\epsilon^2$ noise variance lower bound in Eq.~\ref{eq:mseFromBiasVarianceNoise}; 
while unattainable out of sample, it nonetheless provides an aspirational ideal.

\paragraph*{\em Optimal model-based regularization using the true generative parameters $\theta, \epsilon$} 
We used the value $\lambda_{\text{fp}}$ returned by Algorithm~\ref{alg:iterative-optimal-regularization},
based on a fixed-point iteration solution of the first-order analytical stationarity condition of Eq.~\ref{eq:stationarityOptimalLambda}.
\paragraph*{\em Sample-based regularization using our proposed technique, Algorithm~\ref{alg:sample-based-optimal-regularization}}
We used the value $\lambda_{\text{sfp}}$ returned by Algorithm~\ref{alg:sample-based-optimal-regularization}. This approach is based on estimation of the generative noise standard deviation $\epsilon$ using Eqs.~\ref{eq:epsilonRegEstimate},~\ref{eq:epsilonRegEstimateOverparameterized}
and of the label-generative parameter $\theta$ using a preliminary ridge estimate.
\paragraph*{\em A sample-based version of the signal-to-noise approach of~\cite{DobribanWagnerHighDAsymptotics2018}} 
We used the value $\lambda_{\text{SN}} = d\, \hat{\epsilon}^2 / \Vert \hat{\theta} \Vert^2$ from section 2 of~\cite{DobribanWagnerHighDAsymptotics2018}\footnote{Note 
the factor $\lambda N$ in~\cite{DobribanWagnerHighDAsymptotics2018}, p.254 vs.\ only $\lambda$
in Eq.~\ref{eq:ridgeParameterExpression} (present paper).}, using $\Vert \hat{\theta}\Vert^2$ as a single-point estimate of $\tr \cov(\theta)$, and the same estimates $\hat{\epsilon}, \hat{\theta}$ as in Algorithm~\ref{alg:sample-based-optimal-regularization}.
\paragraph*{\em Default regularization $\lambda_0 = 1$ (e.g.,~\cite{scikit-learn})} 
We confirmed the use of the unscaled value $1$ in the \verb|scikit-learn| \verb|Ridge| class by direct comparison of the \verb|Ridge| model coefficients with the expression of Eq.~\ref{eq:ridgeParameterExpression} over sample data.

\subsubsection{Statistical measures of variation}
For each regularization approach tested, we computed BCa-bootstrapped~\cite{EfronBCaBootstrap1987} $95\%$ confidence intervals for the median regularization strength $\lambda$ produced by that approach, as well as for the mean squared error resulting from that value of $\lambda$, based on $10,\!000$ bootstrap samples.

\subsubsection{Software and hardware platforms}
The evaluation procedures of Algorithms~\ref{alg:experimental-procedure-fixed-X},~\ref{alg:experimental-procedure-random-X} were implemented in the Python programming language,
using the \verb|scikit-learn|~\cite{scikit-learn}, \verb|numpy|~\cite{numpy}, and \verb|scipy|~\cite{scipy} libraries. Visualization used \verb|matplotlib|~\cite{matplotlib}.
Hardware was a Dell Latitude 7450 laptop computer with Intel\textsuperscript{\textregistered} Core\textsuperscript{\texttrademark} Ultra 7 165U processor at $\SI{2.1}{\giga\hertz}$ and $\SI{32}{\giga\byte}$ of RAM.

\subsection{Results in the fixed-$X$ setting}
\label{section:fixed-X-results}

Evaluation in the fixed-X setting followed Algorithm~\ref{alg:experimental-procedure-fixed-X},
using $m_{\theta}=100$ randomly sampled parameters $\theta$, each of which was paired with $m_X=50$ random
$X$ samples, with $m_y=50$ random $y$ samples for each pair $(\theta, X)$.

Bootstrapped confidence intervals for median out-of-sample MSE appear in Table~\ref{table:med-MSE-CI-fixed-X}.
Generalization performance stratifies into three tiers:
model-based Algorithm~\ref{alg:iterative-optimal-regularization} and ideal model-based MSE-minimizing regularization (c.f., the discussion on global vs.\ local optimality in section~\ref{subsection:modelBasedFixedPointOptimality})
tie for best; the sample-based approach of Algorithm~\ref{alg:sample-based-optimal-regularization} and
the signal-to-noise approach based on~\cite{DobribanWagnerHighDAsymptotics2018} have intermediate performance; and default $\lambda = 1$ regularization performs worst. The signal-to-noise approach slightly beats Algorithm~\ref{alg:sample-based-optimal-regularization} at the smallest sample size; the order reverses
at higher sample sizes, with Algorithm~\ref{alg:sample-based-optimal-regularization}'s advantage increasing
with sample size. 

\begin{algorithm}[h!]
\begin{algorithmic}[1]
\Function{EvaluateFixedX}{$N, d, \Sigma, \epsilon, m_{\theta}, m_X, m_y$}
\small
\For{$i_{\theta} = 1, \cdots, m_{\theta}$}
\State $\theta \sim \text{Unif}(\{ v \in \mathbb R^{d \times 1} \mid v^T v = 1 \})$
\For{$i_X = 1, \cdots, m_X$}
\State $X = \Call{GenXData}{N, d, \Sigma}$
\State $U, S, V = \Call{SVD}{X}$
\State $\lambda_{\text{min}}[i_X, i_{\theta}] =$ from Eq.~\ref{eq:lambdaStar}
\State $\text{mse}^*[i_X, i_{\theta}] =$ from Eq.~\ref{eq:mseFromBiasVarianceNoise} with $\lambda_{\text{min}}[i_X, i_{\theta}]$
\State $(\lambda_{\text{fp}}, \text{mse}_{\text{fp}})[i_X, i_{\theta}] = \Call{ModelOptReg}{S, V, \theta, \epsilon}$ 
\label{algline:model-based-call}
\For{$i_y = 1, \cdots, m_y$}
\State $y = \Call{GenYData}{X, \theta, \epsilon}$
\State $\lambda_{\text{sfp}}[i_X, i_{\theta}, i_y] = \Call{SampleOptReg}{X, y}$ 
\label{algline:sample-based-call}
\State $\text{mse}_{\text{sfp}}[i_X, i_{\theta}, i_y] =$ from Eq.~\ref{eq:mseFromBiasVarianceNoise} w.\ $\lambda[i_X, i_{\theta}, i_y]$
\State $\lambda_{\text{SN}}[i_X, i_{\theta}, i_y], \text{mse}_{\text{SN}}[i_X, i_{\theta}, i_y] =$ see~\ref{subsection:regularization-approaches}
\EndFor	
\EndFor 
\EndFor	
\State \Return{BCa-bootstrapped $95\%$ CI for $\lambda$, $\text{mse}$ values}
\normalsize
\EndFunction
\end{algorithmic}
\caption{Experimental evaluation, fixed-X setting. \linebreak
Model-based and sample-based optimal regularization strengths
on lines~\ref{algline:model-based-call},~\ref{algline:sample-based-call}
are as in Algorithms~\ref{alg:iterative-optimal-regularization},~\ref{alg:sample-based-optimal-regularization},
respectively. Hyperparameter values not shown: $\lambda_0=1, \delta=10^{-4}, p=\frac{1}{4}$.
$X, y$ data are generated as in Algorithm~\ref{alg:data-generation-model}.}
\label{alg:experimental-procedure-fixed-X}
\end{algorithm}

\begin{table}[b]
\caption{Median out-of-sample MSE, fixed-X setting.\linebreak $95\%$ BCa bootstrap confidence intervals.
Aspect ratio $\frac{d}{N} = 0.9$. Spiked covariance, noise $\epsilon=1$. Values based on
$10^4$ bootstrap samples over $50 \times 50$ target vectors for each of $100$ generative $\theta$. 
$\lambda_{\text{SN}}$~\cite{DobribanWagnerHighDAsymptotics2018} is best of sample-based methods for $N=20$;
Alg.~\ref{alg:sample-based-optimal-regularization} wins by a growing margin as $N$ increases.}
\vspace{-0.2cm}
\begin{center}
\renewcommand{\arraystretch}{1.2}
\setlength{\tabcolsep}{3pt}
\begin{tabular}{cccc}
\hline
{\normalsize $N$} &{\normalsize $20$} &{\normalsize $100$} &{\normalsize $500$}\\
\hline
optimization-based $\lambda_{\text{min}}$ &$(1.14, 1.16)$ &$(1.15, 1.16)$ &$(1.15, 1.16)$\\
model-based $\lambda_\text{fp}$ (Alg.~\ref{alg:iterative-optimal-regularization}) &$(1.14, 1.16)$ &$(1.15, 1.16)$ &$(1.15, 1.15)$\\
sample-based $\lambda_{\text{sfp}}$ (Alg.~\ref{alg:sample-based-optimal-regularization}) &$(1.17, 1.17)$ &$\mathbf{(1.18, 1.18)}$ &$\mathbf{(1.19, 1.19)}$\\
signal-to-noise $\lambda_{\text{SN}}$~\cite{DobribanWagnerHighDAsymptotics2018} &$\mathbf{(1.16, 1.16)}$ &$(1.24,1.25)$ &$(1.37,1.37)$\\
default $\lambda_0 = 1$ (e.g.,~\cite{scikit-learn}) &$(1.34, 1.34)$ &$(1.49, 1.49)$ &$(1.61, 1.61)$\\
\hline
\end{tabular}
\end{center}
\label{table:med-MSE-CI-fixed-X}
\end{table}

Algorithm~\ref{alg:sample-based-optimal-regularization} attains the strong generalization performance observed in Table~\ref{table:med-MSE-CI-fixed-X} despite the use of regularization strengths $\lambda_{\text{sfp}}$ that differ considerably from the model-based $\lambda_{\text{min}}$, as shown in Table~\ref{table:med-alpha-CI-fixed-X}. 
The difference between the signal-to-noise-based $\lambda_{\text{SN}}$~\cite{DobribanWagnerHighDAsymptotics2018} and $\lambda_{\text{min}}$ at larger sample sizes is considerably greater, however,
consistent with the weaker larger-sample generalization performance observed in Table~\ref{table:med-MSE-CI-fixed-X} for $\lambda_{\text{SN}}$ as compared with $\lambda_{\text{sfp}}$.
We refrain from additional discussion of fixed-X evaluation, because the results are qualitatively very similar to
those in the random-X setting, discussed next.

\begin{table}
\caption{Median regularization strengths, fixed-X setting. $95\%$ BCa bootstrap confidence intervals.
Data aspect ratio $\frac{d}{N} = 0.9$. Spiked covariance model, noise $\epsilon=1$. Values based on
$10^4$ bootstrap samples over $50 \times 50$ target vectors for each of $100$ generative parameters. 
Although the sample-based regularization strength estimates 
$\lambda_{\text{sfp}}, \lambda_0$
are noticeably smaller than the optimal model-based $\lambda_{\text{min}}$, the $\lambda_{\text{sfp}}$ estimate (Algorithm~\ref{alg:sample-based-optimal-regularization}) yields near-optimal MSE 
values (see Table~\ref{table:med-MSE-CI-fixed-X}).}
\vspace{-0.2cm}
\begin{center}
\renewcommand{\arraystretch}{1.2}
\setlength{\tabcolsep}{4pt}
\begin{tabular}{cccc}
\hline
{\normalsize $N$} &{\normalsize $20$} &{\normalsize $100$} &{\normalsize $500$}\\
\hline
model-based $\lambda_{\text{min}}$ &$(17.0,20.0)$ &$(85.9,94.7)$ &$(448,465)$\\
model-based $\lambda_\text{fp}$ (Alg.~\ref{alg:iterative-optimal-regularization}) &$(17.0,19.9)$ &$(85.9,94.7)$ &$(448,465)$\\
sample-based $\lambda_{\text{sfp}}$ (Alg.~\ref{alg:sample-based-optimal-regularization}) &$(9.85,10.0)$ &$(35.8,36.0)$ &$(150,150)$\\
signal-to-noise $\lambda_{\text{SN}}$~\cite{DobribanWagnerHighDAsymptotics2018} &$(10.5, 10.6)$ &$(15.4,15.4)$ &$(19.8,19.9)$\\
default $\lambda_0 = 1$ (e.g.,~\cite{scikit-learn}) &$(1,1)$ &$(1,1)$ &$(1,1)$\\
\hline
\end{tabular}
\end{center}
\label{table:med-alpha-CI-fixed-X}
\end{table}

\subsection{Results in the random-$X$ setting}
\label{section:random-X-results}

Random-X evaluation followed Alg.~\ref{alg:experimental-procedure-random-X}.
We tested $m_{\theta} = 100$ random unit vectors as the label-generative parameter, $\theta$,
using $m_{X,y} = 100$ random $X, y$ samples per $\theta$. 
See Tables~\ref{table:med-MSE-CI-random-X},~\ref{table:med-alpha-CI-random-X},~\ref{table:med-eps-CI-random-X}
and Figs.~\ref{fig:msesN100dVarious},
~\ref{fig:msesN100d90EpsVariousRelative}.
We omit Alg.~\ref{alg:iterative-optimal-regularization} because its
performance is nearly indistinguishable from the globally optimal $\lambda_{\text{min}}$.

\begin{algorithm}
\begin{algorithmic}[1]
\Function{EvaluateRandX}{$N, d, \Sigma, \epsilon, N_{\text{test}}, m_{\theta}, m_{X,y}$}
\small
\For{$i_{\theta} = 1, \cdots, m_{\theta}$}
\State $\theta \sim \text{Unif}(\{ v \in \mathbb R^{d \times 1} \mid v^T v = 1 \})$
\State $X_{\text{test}} = \Call{GenXData}{N_{\text{test}}, d, \Sigma}$
\State $y_{\text{test}} = \Call{GenYData}{X_{\text{test}}, \theta, \epsilon}$
\For{$i_{X,y} = 1, \cdots, m_{X,y}$}
\State $X = \Call{GenXData}{N, d, \Sigma}$
\State $y = \Call{GenYData}{X, \theta, \epsilon}$
\State $U, S, V = \Call{SVD}{X}$
\State $\lambda_{\text{min}}[i_{X,y}, i_{\theta}] = \arg\min_{\lambda} 
\Call{mse}{X_{\text{test}}\hat{\theta}(\lambda), y_{\text{test}}}$, 
\State \ \ where $\hat{\theta}(\lambda) = (X^T X + \lambda I)^{\dagger} X^T y$ (also in~\ref{algline:secondUseThetaHatOfLambdaNotation})
\label{algline:thetaHatIndexNotation}
\State $\text{mse}^*[i_{X,y}, i_{\theta}] = \Call{mse}{X_{\text{test}}\hat{\theta}(\lambda_{\text{min}}[i_{X,y}, i_{\theta}]), y_{\text{test}}}$ 
\State $\lambda_{\text{sfp}}[i_{X,y}, i_{\theta}] = \Call{SampleOptReg}{X, y}$ 
\label{algline:sample-based-call-randX}
\State $\text{mse}_{\text{sfp}}[i_{X,y}, i_{\theta}] = \Call{mse}{X_{\text{test}}\hat{\theta}(\lambda_{\text{sfp}}[i_{X,y}, i_{\theta}]), y_{\text{test}}}$
\label{algline:secondUseThetaHatOfLambdaNotation}
\State $\lambda_{\text{SN}}[i_{X,y}, i_{\theta}], \text{mse}_{\text{SN}}[i_{X,y}, i_{\theta}] =$ see~\ref{subsection:regularization-approaches}
\EndFor 
\EndFor	
\State \Return{BCa-bootstrapped $95\%$ CI for $\lambda$, $\text{mse}$ values}
\normalsize
\EndFunction
\end{algorithmic}
\caption{Experimental evaluation, random-X setting.
Model-based and sample-based optimal regularization strengths
on lines~\ref{algline:model-based-call},~\ref{algline:sample-based-call}
are as in Algorithms~\ref{alg:iterative-optimal-regularization},~\ref{alg:sample-based-optimal-regularization},
respectively. Hyperparameter values not shown: $\lambda_0=1, \delta=10^{-4}, p=\frac{1}{4}$. 
$X, y$ data are generated as in Algorithm~\ref{alg:data-generation-model}.}
\label{alg:experimental-procedure-random-X}
\end{algorithm}

\begin{figure}[h!]
    \centering
    \includegraphics[width=\figurewidth, trim = 0 2 4 2, clip=True]{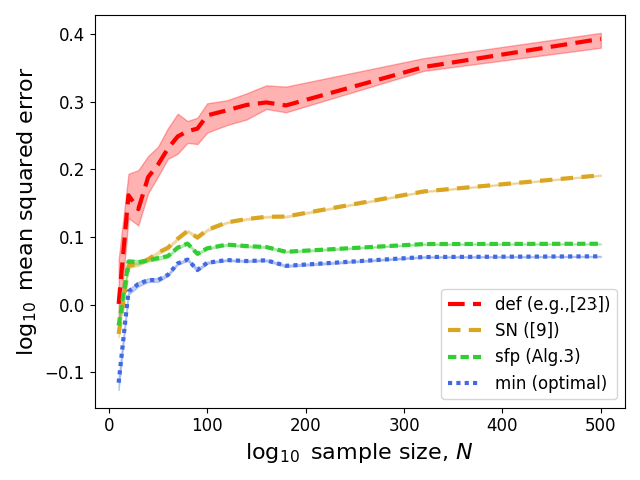}\\
    \caption{Median estimated out-of-sample MSE as a function of sample size, data aspect ratio $0.9$. 
    Spiked covariance model, $\epsilon=1$. 
    Shading shows BCa bootstrap $95\%$ confidence intervals over $100$ 
    repetitions of each of $100$ generative vectors $\theta$; visually perceptible differences are significant. 
    Signal-to-noise~\cite{DobribanWagnerHighDAsymptotics2018} is best at smaller sample sizes, $N < 40$.
    Alg.~\ref{alg:sample-based-optimal-regularization} is best for larger samples, $N > 40$,
    and its generalization benefit over signal-to-noise
    grows with $N$.
    Also see Table~\ref{table:med-MSE-CI-random-X}.}
    \label{fig:msesAspectRatio0.9NVarious}
\end{figure}

\begin{table}[t]
\caption{Median out-of-sample random-X MSE. Covariance as marked. Aspect ratio $\frac{d}{N} = 0.9$, noise $\epsilon=1$; best sample-based results in {\bf bold}.
$95\%$ BCa bootstrap confidence intervals based on $10^4$ bootstrap samples over $N_{\text{test}} = 1000$ 
test examples for ridge parameter estimates over $m_{X,y} = 100$ samples $(X,y)$ for each of $m_{\theta} = 100$ generative $\theta$. See Alg.~\ref{alg:experimental-procedure-random-X}.}
\vspace{-0.2cm}
\begin{center}
\renewcommand{\arraystretch}{1.2}
\setlength{\tabcolsep}{3pt}
\begin{tabular}{cccc}
\hline
{\normalsize $N$} &{\normalsize $20$} &{\normalsize $100$} &{\normalsize $500$}\\
\hline
&{\normalsize Spiked} \\
$\lambda_{\text{min}}\ (\text{Alg.}~\ref{alg:experimental-procedure-random-X})$ &$(1.01, 1.03)$ &$(1.15, 1.16)$ &$(1.17, 1.18)$\\
sample-based $\lambda_{\text{sfp}}$ (Alg.~\ref{alg:sample-based-optimal-regularization}) &$(1.13, 1.15)$ &$\mathbf{(1.21, 1.22)}$ &$\mathbf{(1.23, 1.23)}$\\
signal-to-noise $\lambda_{\text{SN}}$~\cite{DobribanWagnerHighDAsymptotics2018} &$\mathbf{(1.11, 1.13)}$ &$(1.29, 1.29)$ &$(1.55, 1.56)$\\
default $\lambda_0 = 1$ (e.g.,~\cite{scikit-learn}) &$(1.33, 1.53)$ &$(1.81, 1.95)$ &$(2.42, 2.51)$\\
\hline
\end{tabular}
\begin{tabular}{cccc}
&{\normalsize Bulk} \\
$\lambda_{\text{min}}\ (\text{Alg.}~\ref{alg:experimental-procedure-random-X})$ &$(1.21, 1.22)$ &$(1.27, 1.28)$ &$(1.31, 1.32)$\\
sample-based $\lambda_{\text{sfp}}$ (Alg.~\ref{alg:sample-based-optimal-regularization}) &$(1.35, 1.37)$ &$\mathbf{(1.41, 1.42)}$ &$\mathbf{(1.46, 1.47)}$\\
signal-to-noise $\lambda_{\text{SN}}$~\cite{DobribanWagnerHighDAsymptotics2018} &$(1.35, 1.37)$ &$(1.56, 1.58)$ &$(1.93, 1.94)$\\
default $\lambda_0 = 1$ (e.g.,~\cite{scikit-learn}) &$(1.70, 1.94)$ &$(3.03, 3.39)$ &$(5.12, 5.48)$\\
\hline
\end{tabular}
\end{center}
\label{table:med-MSE-CI-random-X}
\end{table}

\begin{table}
\caption{Median random-X regularization strengths. Covariance as marked. Aspect ratio $\frac{d}{N} = 0.9$, noise $\epsilon=1$.
$95\%$ BCa bootstrap confidence intervals based on $10^4$ bootstrap samples over $N_{\text{test}} = 1000$ 
test examples for ridge parameter estimates derived from $m_{X,y} = 100$ samples $(X,y)$ for each of $m_{\theta} = 100$ generative $\theta$. See Algorithm~\ref{alg:experimental-procedure-random-X}.}
\vspace{-0.2cm}
\begin{center}
\renewcommand{\arraystretch}{1.2}
\setlength{\tabcolsep}{4pt}
\begin{tabular}{cccc}
\hline
{\normalsize $N$} &{\normalsize $20$} &{\normalsize $100$} &{\normalsize $500$}\\
\hline
&{\normalsize Spiked} &&\\
$\lambda_{\text{min}}\ (\text{Alg.}~\ref{alg:experimental-procedure-random-X})$ &$(16.1, 17.4)$ &$(88.4, 91.1)$ &$(449, 457)$\\
sample-based $\lambda_{\text{sfp}}$ (Alg.~\ref{alg:sample-based-optimal-regularization}) &$(9.9, 10.3)$ &$(35.6,36.4)$ &$(147,149)$\\
signal-to-noise $\lambda_{\text{SN}}$~\cite{DobribanWagnerHighDAsymptotics2018} &$(10.4, 10.7)$ &$(15.3, 15.5)$ &$(19.8,19.9)$\\
default $\lambda_0 = 1$ (e.g.,~\cite{scikit-learn}) &$(1,1)$ &$(1,1)$ &$(1,1)$\\
\hline
\end{tabular}
\begin{tabular}{cccc}
&{\normalsize Bulk} &&\\
$\lambda_{\text{min}}\ (\text{Alg.}~\ref{alg:experimental-procedure-random-X})$ &$(14.0, 14.8)$ &$(89.4, 92.3)$ &$(446, 452)$\\
sample-based $\lambda_{\text{sfp}}$ (Alg.~\ref{alg:sample-based-optimal-regularization}) &$(7.8, 8.2)$ &$(27.6, 28.4)$ &$(107, 109)$\\
signal-to-noise $\lambda_{\text{SN}}$~\cite{DobribanWagnerHighDAsymptotics2018} &$(6.6, 6.8)$ &$(13.1, 13.3)$ &$(27.0, 27.2)$\\
default $\lambda_0 = 1$ (e.g.,~\cite{scikit-learn}) &$(1,1)$ &$(1,1)$ &$(1,1)$\\
\hline
\end{tabular}
\end{center}
\label{table:med-alpha-CI-random-X}
\end{table}

\begin{table}
\caption{Median random-X generative parameter estimates. Covariance as marked. Aspect ratio $\frac{d}{N} = 0.9$, noise $\epsilon=1$, generative parameter norm $\Vert \theta \Vert = 1$.
$95\%$ BCa bootstrap confidence intervals based on $10^4$ bootstrap samples over $N_{\text{test}} = 1000$ 
test examples for ridge parameter estimates derived from $m_{X,y} = 100$ data samples $(X,y)$ for each of $m_{\theta} = 100$ true generative parameters, $\theta$. See Algorithm~\ref{alg:experimental-procedure-random-X}.}
\vspace{-0.2cm}
\begin{center}
\renewcommand{\arraystretch}{1.2}
\setlength{\tabcolsep}{4pt}
\begin{tabular}{cccc}
\hline
{\normalsize $N$} &{\normalsize $20$} &{\normalsize $100$} &{\normalsize $500$}\\
\hline
&{\normalsize Spiked} &&\\
rank ${r_p}$ (Eq.~\ref{eq:regRankExpression}, $p=0.25$) &$(12.4, 12.4)$ &$(71.8, 71.8)$ &$(393, 393)$\\
noise $\hat{\epsilon}_p\ (\text{Eq.}~\ref{eq:epsilonRegEstimate}, p=0.25)$ &$(1.07, 1.08)$ &$(1.13, 1.13)$ &$(1.13, 1.14)$\\
norm $\Vert \hat{\theta} \Vert$ (Alg.~\ref{alg:sample-based-optimal-regularization}, $\lambda_0=1$) &$(1.39, 1.40)$ &$(2.72, 2.73)$ &$(5.41, 5.42)$\\
\hline
\end{tabular}
\begin{tabular}{cccc}
&{\normalsize Bulk} &&\\
rank ${r_p}$ (Eq.~\ref{eq:regRankExpression}, $p=0.25$) &$(15.8, 15.8)$ &$(85.5, 85.6)$ &$(442, 442)$\\
noise $\hat{\epsilon}_p\ (\text{Eq.}~\ref{eq:epsilonRegEstimate}, p=0.25)$ &$(0.98, 1.00)$ &$(1.01, 1.02)$ &$(1.0,1.0)$\\
norm $\Vert \hat{\theta} \Vert$ (Alg.~\ref{alg:sample-based-optimal-regularization}, $\lambda_0=1$) &$(1.60, 1.62)$ &$(2.64, 2.66)$ &$(4.07, 4.09)$\\
\hline
\end{tabular}
\end{center}
\label{table:med-eps-CI-random-X}
\end{table}

\subsection{Discussion}
\label{subsection:discussion}

\subsubsection{Generalization across sample sizes and aspect ratios}
Table~\ref{table:med-MSE-CI-random-X} compares random-$X$ generalization of Alg.~\ref{alg:sample-based-optimal-regularization} against default regularization $\lambda_0 = 1$~\cite{scikit-learn} and a sample-based version of the signal-to-noise approach of~\cite{DobribanWagnerHighDAsymptotics2018} (section~\ref{subsection:regularization-approaches}). Also see Fig.~\ref{fig:msesAspectRatio0.9NVarious}.
For spiked covariance, the signal-to-noise approach performs best at smaller sample sizes ($N \lesssim 30$). Alg.~\ref{alg:sample-based-optimal-regularization} is the best for moderate and larger sample sizes, 
and its advantage over signal-to-noise grows with sample size.

Table~\ref{table:med-alpha-CI-random-X} links performance of these approaches to the associated regularization strengths relative to the optimal $\lambda_{\text{min}}$: the values $\lambda_{\text{sfp}}$ 
produced by Alg.~\ref{alg:sample-based-optimal-regularization} are much lower than the optimal values $\lambda_{\text{min}},$ but the values $\lambda_{\text{SN}}$ produced by the signal-to-noise approach are lower still, contributing to its
weaker performance as compared with Alg.~\ref{alg:sample-based-optimal-regularization}. Similarly, default regularization $\lambda=1$ performs poorly because its
data-blind scale is too small for the underlying data distribution.

Fig.~\ref{fig:msesN100dVarious} illustrates the fact that the performance ranking seen in Table~\ref{table:med-MSE-CI-random-X} in the underparameterized case persists across data aspect ratios, including the overparameterized regime, with Algorithm~\ref{alg:sample-based-optimal-regularization} ranking best among sample-based approaches,
followed by the signal-to-noise approach, with default regularization $\lambda=1$ placing a distant last.

\subsubsection{Generalization across noise levels}
\label{subsection:mseVsNoise}

Out-of-sample error grows with noise level $\epsilon$ as it does with sample size. Performance differences
among approaches become clearer by considering error relative to the baseline associated with the
optimal regularization strength, $\lambda_{\text{min}}$. See Fig.~\ref{fig:msesN100d90EpsVariousRelative}.
Alg.~\ref{alg:sample-based-optimal-regularization} generalizes best among approaches except at low noise levels 
$\epsilon \lesssim 0.03$.
Differences are magnified at larger sample sizes $N > 100$ (Table~\ref{table:mseVsNoiseCINvariousTableAppendix}),
and reduced at lower sizes, largely becoming insignificant when $N < 50$.

\begin{figure}[h!]
    \centering
    \includegraphics[width=\figurewidth, trim = 0 2 4 2, clip=True]{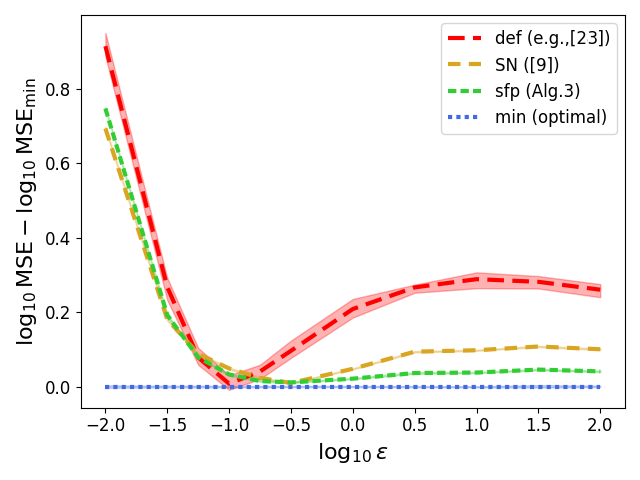}\\
    \caption{Median estimated out-of-sample MSE relative to optimal as a function of noise amplitude, $\epsilon$. 
    Spiked covariance model, $N=100$, $d=90$.
    Shading shows BCa bootstrap $95\%$ confidence intervals over $100$ 
    repetitions of each of $100$ generative vectors $\theta$; visually perceptible differences are significant. 
    Signal-to-noise~\cite{DobribanWagnerHighDAsymptotics2018} is best at low noise levels $\epsilon \lesssim 0.03$;
    default $\lambda=1$ best for $\epsilon \approx 0.1$; 
    Alg.~\ref{alg:sample-based-optimal-regularization} best for moderate to high noise, 
    $\epsilon \gtrsim 0.3$. Similar results with less vertical compression (not shown) 
    occur for bulk covariance model, and for larger sample sizes. At smaller sample sizes 
    (not shown, e.g., $N=20$), Alg.~\ref{alg:sample-based-optimal-regularization} is best at noise levels 
    $\epsilon \lesssim0.03$ and equivalent to signal-to-noise~\cite{DobribanWagnerHighDAsymptotics2018}
    for $\epsilon \gtrsim 0.03$.}
    \label{fig:msesN100d90EpsVariousRelative}
\end{figure}

\subsubsection{Suboptimality of the signal-to-noise approach}

The model-based signal-to-noise regularization strength $\lambda_{\text{SN}} = d\, \epsilon^2 / \Vert \theta \Vert^2$ of~\cite{DobribanWagnerHighDAsymptotics2018} would provide strong generalization if used with the true generative parameter values $\epsilon=1, \Vert \theta \Vert=1$ used in Table~\ref{table:med-alpha-CI-random-X}: the resulting strength $\lambda_{\text{SN}} = d$ lies within the confidence interval for the optimal strength $\lambda_{\text{min}}$ for the larger sample sizes $N=100, 500$, and remains
close to that interval for the smaller size $N=20$. 

In the sample-based context, the signal-to-noise approach suffers from its sensitivity to the
ridge-estimated parameter norm, $\Vert \hat{\theta} \Vert$: 
Table~\ref{table:med-eps-CI-random-X} shows that the latter values are 
much higher (though less so than an unregularized estimate) than the true norm of $1$,
yielding $\lambda_{\text{SN}}$ values far below optimal. 

Using a positive exponent $p=0.25$ in the regularized rank computation of Eq.~\ref{eq:epsilonRegEstimate} compensates somewhat by increasing the estimated noise level $\hat{\epsilon}_p$, 
slightly raising the estimated signal-to-noise ratio $d \,{\hat{\epsilon}}^2 / \Vert \hat{\theta} \Vert^2$.
This effect is minor, however: the regularized rank
${r_p}$ is only $\approx 20\%$ smaller than the nominal rank $d$ at higher sample sizes (Table~\ref{table:med-eps-CI-random-X}), 
not nearly enough for the resulting larger $\hat{\epsilon}_{p}$ to offset the excess in the
$\Vert \hat{\theta} \Vert$ norm estimate, 
yielding small regularization strengths $\lambda_{\text{SN}}$ and weaker
generalization for $\lambda_{\text{SN}}$ in Table~\ref{table:med-MSE-CI-random-X}.

Unlike signal-to-noise, Alg.~\ref{alg:sample-based-optimal-regularization}
accounts for the relationship between the generative parameter $\theta$ and the data covariance
geometry of $X$ (via the terms $v^T_j \theta$ in Eq.~\ref{eq:lambdaUpdateOperator}).
This likely also contributes to the stronger generalization of Alg.~\ref{alg:sample-based-optimal-regularization}.

\begin{figure}[h!]
    \centering
    \includegraphics[width=\figurewidth, trim = 0 2 4 2, clip=True]{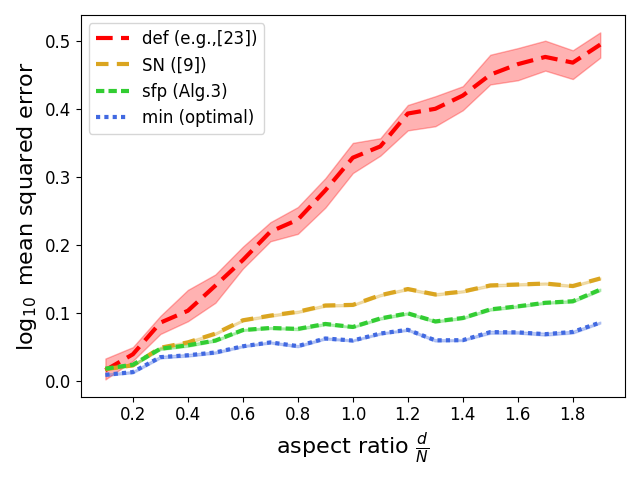}\\
    \caption{Median estimated out-of-sample MSE as a function of data aspect ratio. 
    Spiked covariance model, $N=100$, $\epsilon=1$. 
    Shading shows BCa bootstrap $95\%$ confidence intervals over $100$ 
    repetitions of each of $100$ generative vectors $\theta$; visually perceptible differences are significant. 
    Algorithm~\ref{alg:sample-based-optimal-regularization} is best for aspect ratios $\frac{d}{N} > 0.4$.
    Median ratio of MSE difference between $\lambda_{\text{SN}}$, $\lambda_{\text{min}}$ to MSE difference
    between $\lambda_{\text{sfp}}$, $\lambda_{\text{min}}$ is $1.9$.} 
    \label{fig:msesN100dVarious}
\end{figure}

\section{Conclusions}
\label{section:conclusions}

We present a fixed-point iteration approach for finding the optimal ridge regularization penalty in the fixed $X$ setting, assuming data generated by a distribution of bounded covariance and additive isotropic noise of finite
variance. 
We prove that the proposed technique converges for sufficiently small values of the noise variance. We further propose an approach for sample-based estimation of the generative parameters needed for the fixed-point computation. Experimental results show that the combined sample-based regularization approach of Algorithm~\ref{alg:sample-based-optimal-regularization} attains near-optimal generalization across a wide range of data aspect ratios and moderate to high noise levels, even in the random $X$ setting, better than both a data-agnostic regularization strength and a sample-based version of the
asymptotically optimal signal-to-noise regularization strength of~\cite{DobribanWagnerHighDAsymptotics2018}.
The generalization advantage of the proposed approach grows with sample size and noise level.
There is room for future work on the sample-based estimation procedure, extreme aspect ratios $d/N \gg 1$,
and out-of-distribution generalization.

\subsubsection*{Author roles.}
JT contributed to the analysis, proof of Theorem~\ref{thm:contractivenessOfLambdaOperator}, and experiments in the underparameterized case. SA conceived the research and analytical approach, carried out the analysis and proof
of Theorem~\ref{thm:contractivenessOfLambdaOperator} in the general case, designed and carried out the experimental
evaluation, and wrote the paper.

\bibliographystyle{IEEEtran}

\small
\bibliography{bibcontrol, optimalRidgeRegression}
\normalsize

\newpage

\renewcommand{\thetable}{%
  \thesection.%
  \arabic{table}%
}
\setcounter{table}{0} 

\renewcommand{\thefigure}{%
  \thesection.%
  \arabic{figure}%
}
\setcounter{figure}{0}
 
\appendices
\section{Supplementary information for the paper\\
{\normalfont Optimal ridge regularization revisited}}
\label{section:appendix}

\bigskip
\subsection{Predictive bias and variance: derivation details}
\label{subsection:biasVarianceSupplement}

We describe how the parameter estimate of Eq.~\ref{eq:ridgeParameterExpression} leads,
via the data-generative model of Eq.~\ref{eq:dataGenerativeModel} and the SVD $X = U S V^T$, to the
predictive bias and variance expressions of Eqs.~\ref{eq:ridgePredictionBias},~\ref{eq:ridgePredictionVar}, 
regardless of data aspect ratio. 
We assume that the unlabeled data matrix $X$ has full rank.

Following Appendix B of~\cite{patil2024optimal}, the pseudo-inverse expression of Eq.~\ref{eq:ridgeParameterExpression} for the ridge parameter estimate $\hat{\theta}$ can be expressed as a true inverse, with slightly
different forms in the underparameterized and overparameterized cases. We examine these cases
separately in~\ref{subsection:underparameterized-details} and~\ref{subsection:overparameterized-details} below. 

In both cases,
we arrive below at Eq.~\ref{eq:ridgeParameterSVD}, which immediately yields Eq.~\ref{eq:ridgePredictionSVD}.
The form of the block matrix in each of Eqs.~\ref{eq:blocksUnderparameterized},~\ref{eq:blocksOverparameterized} 
then leads to the summation expressions in Eqs.~\ref{eq:ridgePredictionBias}-~\ref{eq:ridgePredictionVar}
involving only the first $\mindN$ singular values and vectors,
where $\mindN$ denotes the minimum of $d,N$ as in the main text.

\subsubsection{Underparameterized case, $d < N$}
\label{subsection:underparameterized-details}
In this case, the pseudoinverse in Eq.~\ref{eq:ridgeParameterExpression} is already an inverse.
Using the SVD $X = U S V^T$ and (self-)orthogonality of $U$, $V$, we 
express the ridge-estimated parameter vector as follows:
\begin{align}
\hat{\theta} &= (V S^T S V^T + \lambda I_d)^{-1} V S^T U^T y
\label{eq:underparameterizedParameterVector}
\end{align}
Since $S$ has size $N \times d$, the product $S^T S$ is the diagonal $d \times d$ matrix $\text{diag}(\sigma^2_1, \cdots, \sigma^2_d)$ of squared singular values of $X$; the term $S^T$ at far right of Eq.~\ref{eq:underparameterizedParameterVector} has zeros in its last $N-d$
columns, however. This yields Eq.~\ref{eq:blocksUnderparameterized},
since $V$ is orthogonal. 
\begin{align}
\hat{\theta} &= V \text{diag}( (\sigma^2_j + \lambda)^{-1} ) S^T U^T y
\label{eq:blocksUnderparameterized}
\\
&=
V
\begin{bmatrix}
\begin{array}{ccc}
\frac{\sigma_1}{\sigma^2_1 + \lambda} & & \\
 &\ddots &  \\
 & & \frac{\sigma_d}{\sigma^2_d + \lambda}
\end{array}
\vline
&\begin{matrix}
\\
\scalebox{3}{0}\\
\\
\end{matrix}
\end{bmatrix}
U^T y
\nonumber
\end{align}
Substituting the expression $y = X \theta + \epsilon z$ from Eq.~\ref{eq:b} and using orthogonality of $U$, 
we find the following, where parentheses denote the respective $d \times d$ diagonal matrices. 
\begin{align}
\hat{\theta}
&=
V
\left (
\frac{\sigma^2_j}{\sigma^2_j + \lambda}
\right )
V^T \theta
\ + \ 
\epsilon
V
\begin{bmatrix}
\left (
\frac{\sigma_j}{\sigma^2_j + \lambda}
\right )
\ \vline
\ \begin{matrix}
\\
\scalebox{2}{0}\\
\\
\end{matrix}
\end{bmatrix}
U^T z
\label{eq:blocksUnderparameterized2}
\end{align}
The vector of predicted labels $\hat{y} = X \hat{\theta}$ therefore satisfies:
\begin{align}
\hat{y}
&=
U
\begin{bmatrix}
\left (
\frac{\sigma^3_j}{\sigma^2_j + \lambda}
\right )
\\
\\
\begin{array}{ccc}
\cline{1-3}
\\
&\scalebox{2}{0}&
\end{array}
\end{bmatrix}
V^T \theta
\ + \ 
\epsilon
U
\begin{bmatrix}
\begin{array}{ccc}
\left (
\frac{\sigma^2_j}{\sigma^2_j + \lambda}
\right )
&\vline
&\begin{matrix}
\\
\scalebox{2}{0}\\
\\
\end{matrix}
\\
\cline{1-3}
\begin{matrix}
&\scalebox{2}{0}
\end{matrix}
&\vline
&\begin{matrix}
\\
\scalebox{2}{0}\\
\\
\end{matrix}
\end{array}
\end{bmatrix}
U^T z
\label{eq:blocksUnderparameterized3}
\end{align}
Eq.~\ref{eq:blocksUnderparameterized3} yields Eqs.~\ref{eq:ridgePredictionBias},~\ref{eq:ridgePredictionVar}
in the underparameterized case.

\subsubsection{Overparameterized case, $d \ge N$}
\label{subsection:overparameterized-details}

In this case, the expression in Eq.~\ref{eq:overparameterizedInverseExpression} below holds 
(see~\cite{patil2024optimal}, appendix~B, Eq.~(15),
noting the factor $N$ difference between the convention in~\cite{patil2024optimal} and the present paper).
\begin{align}
\hat{\theta} &= X^T (X X^T + \lambda I_N)^{-1}  y
\label{eq:overparameterizedInverseExpression}
\end{align}
Using the SVD $X = USV^T$ and orthogonality of $V$, the parameter vector
can be expressed as in Eq.~\ref{eq:blocksOverparameterized}.
\begin{align}
\hat{\theta} &= V S^T U^T (U S S^T U^T + \lambda I_N)^{-1}  y \nonumber
\\
&=
V
\begin{bmatrix}
\begin{array}{ccc}
\frac{\sigma_1}{\sigma^2_1 + \lambda} & & \\
 &\ddots &  \\
 & & \frac{\sigma_d}{\sigma^2_d + \lambda}\\
\\
\cline{1-3}
\end{array}
\\
\begin{array}{ccc}
\\
&\scalebox{3}{0}&\\
\\
\end{array}
\end{bmatrix}
U^T y
\label{eq:blocksOverparameterized}
\end{align}
Substituting $y = X\theta + \epsilon z$ from Eq.~\ref{eq:b}, we obtain Eq.~\ref{eq:blocksOverparameterized2};
parenthesized expressions represent $d \times d$ diagonal matrices.
\begin{align}
\hat{\theta} 
&=
V
\begin{bmatrix}
\begin{array}{ccc}
\left (
\frac{\sigma^2_j}{\sigma^2_j + \lambda}
\right )
&\vline
&\begin{matrix}
\\
\scalebox{2}{0}\\
\\
\end{matrix}
\\
\cline{1-3}
\begin{matrix}
&\scalebox{2}{0}
\end{matrix}
&\vline
&\begin{matrix}
\\
\scalebox{2}{0}\\
\\
\end{matrix}
\end{array}
\end{bmatrix}
V^T \theta
\ + \ 
\epsilon
V
\begin{bmatrix}
\left (
\frac{\sigma_j}{\sigma^2_j + \lambda}
\right )
\\
\\
\begin{array}{ccc}
\cline{1-3}
\\
&\scalebox{2}{0}&
\end{array}
\end{bmatrix}
U^T z
\label{eq:blocksOverparameterized2}
\end{align}
The vector of predicted labels $\hat{y} = X \hat{\theta}$ takes the form below:
\begin{align}
\label{eq:blocksOverparameterized3}
\hat{y} 
&=
U
\begin{bmatrix}
\left (
\frac{\sigma^3_j}{\sigma^2_j + \lambda}
\right )
\ \vline
\ \begin{matrix}
\\
\scalebox{2}{0}\\
\\
\end{matrix}
\end{bmatrix}
V^T \theta
\ + \ 
\epsilon
U
\left (
\frac{\sigma^2_j}{\sigma^2_j + \lambda} 
\right )
U^T z
\end{align}
Eq.~\ref{eq:blocksOverparameterized3} yields Eqs.~\ref{eq:ridgePredictionBias},~\ref{eq:ridgePredictionVar}
in the overparameterized case.

\subsection{MSE stationarity condition}
\label{subsection:mse-stationarity-appendix}

We derive the first-order stationarity condition in Eq.~\ref{eq:stationarityOptimalLambda}.
Using Eqs.~\ref{eq:ridgePredictionBias}-\ref{eq:mseFromBiasVarianceNoise}, minimization of expected
out-of-sample mean squared error requires the following condition.
\begin{align}
\frac{d}{d\lambda}
\left ( 
\lambda^2 \sum_{j=1}^{\mindN} 
\left ( 
\frac{\sigma_j v^T_j \theta}{\sigma^2_j + \lambda}
\right )^{\!2}
+\ 
\epsilon^2
\sum_{j=1}^{\mindN}
\left (
\frac{\sigma^2_j}{\sigma^2_j + \lambda}
\right )^{\!2}
\right )
=
0
\end{align}
We differentiate directly, move the $\epsilon^2$ term to the right of the equation, and eliminate the constant factor $2$ throughout. 
\begin{align}
\lambda \sum_{j=1}^{\mindN} 
\left ( 
\frac{\sigma_j v^T_j \theta}{\sigma^2_j + \lambda}
\right )^{\!2}
- 
\lambda^2 \sum_{j=1}^{\mindN} 
\frac{(\sigma_j v^T_j \theta)^2}{(\sigma^2_j + \lambda)^3}
\ =\ 
\epsilon^2
\sum_{j=1}^{\mindN}
\frac{\sigma^4_j}{(\sigma^2_j + \lambda)^3}
\end{align}
Combining the two sums on the left into one having as its $j$th summand the product of $1 - \lambda/(\sigma^2_j + \lambda)$ with the $j$th term of the first (then simplifying), we obtain Eq.~\ref{eq:stationarityOptimalLambda}.

\subsection{Supplementary experimental data}
\label{section:supplementaryDataAppendix}

\subsubsection{Convergence of bootstrap calculations}
\label{subsection:bootstrapConvergenceAppendix}

BCa bootstrap confidence intervals are observed to converge once median aggregation includes $\approx 100$ parameter vectors for the sample sizes tested; see Figs.~\ref{fig:mseCIConvergence}-\ref{fig:lambdaCIConvergence}.
While a small amount of variation remains, the value $m_{\theta} = 100$ strikes a good balance between convergence
and computation time. 

\begin{figure}[H]
    \centering
    \includegraphics[width=\figurewidth, trim = 0 2 0 10, clip=True]{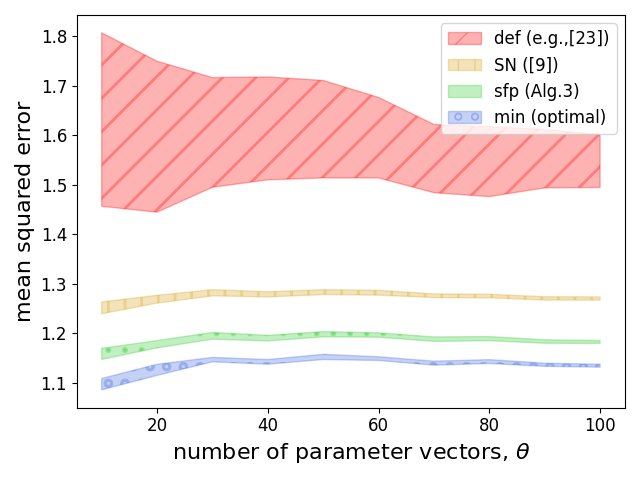}\\
    \caption{Convergence of MSE BCa bootstrap $95\%$ confidence intervals (CI) with number of generative parameter vectors. 
    CI and relative ordering are stable after $\approx 100$ parameter
    vectors. Spiked covariance, $(N, d) = (200, 120)$, noise $\epsilon=1$.}
    \label{fig:mseCIConvergence}
\end{figure}

\begin{figure}[h!]
    \centering
    \includegraphics[width=\figurewidth, trim = 0 2 0 10, clip=True]{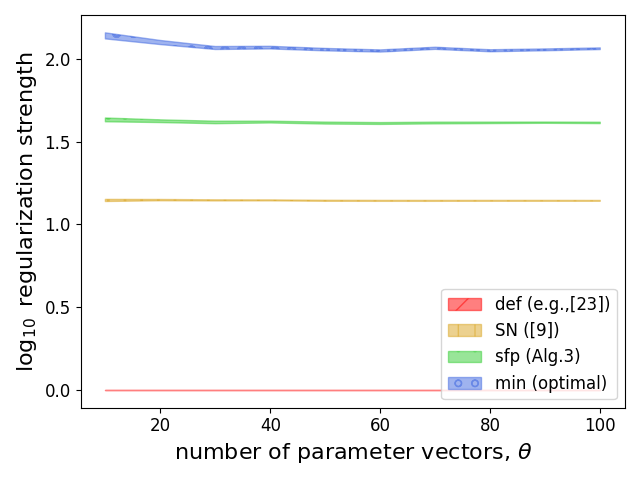}\\
    \caption{Convergence of ridge penalty BCa bootstrap $95\%$ confidence intervals (CI) with number of generative
    vectors. CI and relative ordering are stable after $\approx 100$ parameter vectors. 
    Spiked covariance, $(N, d) = (200, 120)$, noise $\epsilon=1$.}
    \label{fig:lambdaCIConvergence}
\end{figure}

\subsubsection{Performance across sample sizes and aspect ratios}

Tables~\ref{table:mseVsSampleSizeCITableAppendix}-\ref{table:mseVsAspectRatioCITableAppendix} support
the statement that sample data-based regularization via either Alg.~\ref{alg:sample-based-optimal-regularization}
or the sample-based version of the signal-to-noise approach of~\cite{DobribanWagnerHighDAsymptotics2018} is beneficial for all but the smallest sample sizes and aspect ratios. Furthermore, these tables show as stated in the main text that Alg.~\ref{alg:sample-based-optimal-regularization} performs better than signal-to-noise in most cases, 
with a performance advantage that increases with sample size and with aspect ratio.
Figs.~\ref{fig:mseBoxplotsGaussN100d90Eps1},~\ref{fig:msesN100d90Eps1SingleTheta} illustrate the typical ordering of median MSE and regularization strengths among approaches.

\begin{figure}[h!]
    \centering
    \includegraphics[width=\figurewidth, trim = 0 2 4 10, clip=True]{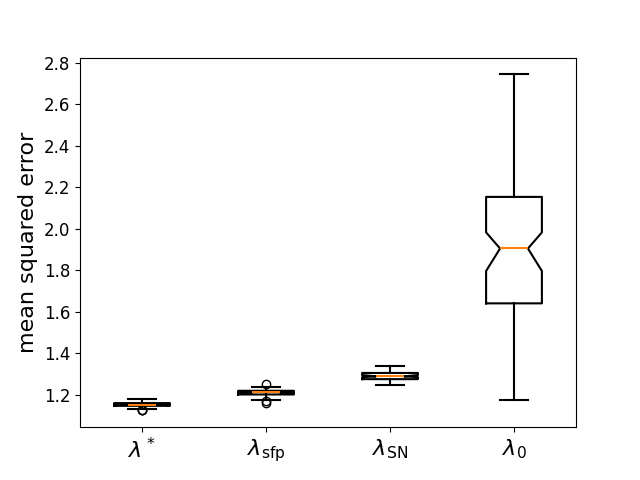}\\
    \caption{Median estimated out-of-sample MSE. 
    Spiked covariance model, $(N, d) = (100, 90)$, noise $\epsilon=1$. Notches show BCa bootstrap $95\%$
    confidence intervals. Values aggregated over $100$ random-X repetitions for each of $100$
    generative $\theta$. Notation is as in section~\ref{subsection:regularization-approaches}.
    Also see Table~\ref{table:med-MSE-CI-random-X}.}
    \label{fig:mseBoxplotsGaussN100d90Eps1}
\end{figure}

\begin{figure}[h!]
    \centering
    \includegraphics[width=\figurewidth, trim = 0 2 4 10, clip=True]{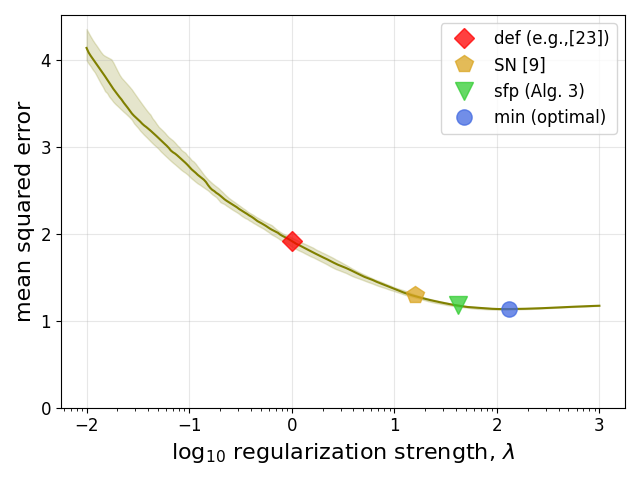}\\
    \caption{Median estimated out-of-sample MSE as a function of regularization strength for a sample
    generative parameter vector, $\theta$. 
    Spiked covariance model, $N=200$, $d=180$, noise $\epsilon=1$. Shading shows BCa bootstrap $95\%$
    confidence intervals. Values aggregated over $100$ random-X repetitions.}
    \label{fig:msesN100d90Eps1SingleTheta}
\end{figure}

\begin{table}[h!]
\caption{Median out-of-sample random-X mean squared error across sample sizes (see Fig.~\ref{fig:msesAspectRatio0.9NVarious}); best sample-based results in {\bf bold}. Spiked covariance model, aspect ratio $\frac{d}{N} = 0.9$, noise $\epsilon=1$. BCa bootstrap $95\%$ confidence intervals based on $10^4$ bootstrap samples. Values are aggregated over $100$ random-X repetitions for each of $100$ generative parameters $\theta$ as in Algorithm~\ref{alg:experimental-procedure-random-X}. $\lambda_{\text{sfp}}$ and $\lambda_{\text{SN}}$ outperform $\lambda_0=1$ for $N \ge 20$;
$\lambda_{\text{sfp}}$ is the best performer for most sample sizes, except for $N \approx 20$, where $\lambda_{\text{SN}}$ is best.}
\centering
\begin{tabular}{ccccc}
{\normalsize $N$}		&{\normalsize $\lambda_{\text{min}}$}		&{\normalsize $\lambda_{\text{sfp}}$}	&{\normalsize $\lambda_{\text{SN}}$}	&{\normalsize $\lambda_0$}\\
\hline
10		&(1.02, 1.02) & (1.04, 1.04) & (1.04, 1.04) & (1.00, 1.08)\\
20     	&(1.04, 1.05) & (1.15, 1.17) & {\bf (1.13, 1.15)} & (1.34, 1.56)\\
30		&(1.06, 1.08) & (1.15, 1.16) & (1.14, 1.16) & (1.31, 1.57)\\
40      	&(1.08, 1.09) & (1.16, 1.17) & (1.16, 1.17) & (1.47, 1.66)\\
60		&(1.10, 1.11) & {\bf (1.17, 1.18)} & (1.21, 1.22) & (1.64, 1.82)\\
80      	&(1.16, 1.17) & {\bf (1.23, 1.24)} & (1.28, 1.29) & (1.73, 1.87)\\
100		&(1.15, 1.16) & {\bf (1.21, 1.22)} & (1.28, 1.29) & (1.80, 1.99)\\
120		&(1.16, 1.17) & {\bf (1.22, 1.23)} & (1.32, 1.33) & (1.84, 2.00)\\
160    	&(1.16, 1.17) & {\bf (1.21, 1.22)} & (1.35, 1.35) & (1.94, 2.11)\\
320     	&(1.17, 1.18) & {\bf (1.23, 1.23)} & (1.47, 1.47) & (2.22, 2.32)\\
500		&(1.18, 1.18)`& {\bf (1.23, 1.23)}	& (1.55, 1.56) & (2.40, 2.52)\\
\hline
\end{tabular}
\label{table:mseVsSampleSizeCITableAppendix}
\end{table}

\begin{table}[h!]
\caption{Median out-of-sample random-X MSE across aspect ratios (see Fig.~\ref{fig:msesN100dVarious}); best sample-based results in {\bf bold}. Spiked covariance, noise $\epsilon=1$. BCa bootstrap $95\%$ confidence intervals based on $10^4$ bootstrap samples. Values aggregated over $100$ random-X repetitions for each of $100$ generative parameters $\theta$ as in Alg.~\ref{alg:experimental-procedure-random-X}. Performance differences grow with aspect ratio and sample size. Alg.~\ref{alg:sample-based-optimal-regularization} is significantly better than signal-to-noise when data dimensionality $d \gtrsim 40$; the two perform statistically equally in nearly all other cases.}
\setlength{\tabcolsep}{6pt}
\centering
\begin{tabular}{ccccc}
{\normalsize $d/N$}		&{\normalsize $\lambda_{\text{min}}$}	&{\normalsize $\lambda_{\text{sfp}}$}	&{\normalsize $\lambda_{\text{SN}}$}	&{\normalsize $\lambda_0$}\\
\hline
\\&&{\normalsize $N = 50$}\\
0.1		& (0.99, 1.00) & (1.03, 1.04) & (1.03, 1.04) & (0.97, 1.09)\\
0.2       & (1.07, 1.08) & (1.11, 1.12) & {\bf (1.10, 1.11)} & (1.08, 1.22)\\
0.3       & (1.03, 1.04) & (1.07, 1.08) & (1.07, 1.08) & (1.06, 1.14)\\
0.4       & (1.07, 1.08) & (1.13, 1.14) & (1.12, 1.14) & (1.20, 1.32)\\
0.5       & (1.07, 1.08) & (1.13, 1.14) & (1.13, 1.14) & (1.25, 1.35)\\
0.6       & (1.07, 1.08) & (1.13, 1.14) & (1.13, 1.15) & (1.32, 1.41)\\
0.7       & (1.11, 1.12) & {\bf (1.17, 1.18)} & (1.18, 1.19) & (1.39, 1.52)\\
0.8       & (1.03, 1.04) & {\bf (1.10, 1.11)} & (1.12, 1.13) & (1.45, 1.71)\\
0.9       & (1.10, 1.11) & {\bf (1.17, 1.18)} & (1.19, 1.20) & (1.47, 1.72)\\
1.0       & (1.12, 1.14) & {\bf (1.20, 1.22)} & (1.23, 1.24) & (1.71, 1.94)\\
1.1       & (1.12, 1.13) & {\bf (1.21, 1.22)} & (1.24, 1.25) & (1.85, 2.02)\\
1.2       & (1.12, 1.13) & {\bf (1.21, 1.22)} & (1.24, 1.25) & (1.80, 2.10)\\
1.3       & (1.11, 1.12) & {\bf (1.21, 1.23)} & (1.25, 1.26) & (1.86, 2.12)\\
1.4       & (1.16, 1.17) & {\bf (1.27, 1.28)} & (1.30, 1.31) & (2.09, 2.32)\\
1.5       & (1.14, 1.15) & {\bf (1.26, 1.27)} & (1.29, 1.30) & (2.07, 2.49)\\
1.6       & (1.13, 1.14) & {\bf (1.26, 1.28)} & (1.29, 1.30) & (2.20, 2.55)\\
1.7       & (1.15, 1.16) & {\bf (1.28, 1.29)} & (1.30, 1.31) & (2.28, 2.59)\\
1.8       & (1.15, 1.16) & {\bf (1.29, 1.30)} & (1.30, 1.31) & (2.39, 2.73)\\
1.9       & (1.19, 1.19) & (1.32, 1.34) & (1.32, 1.34) & (2.21, 2.51)\\
\hline
\\&&{\normalsize $N = 100$}\\
0.1&		(1.02, 1.02)	& (1.04, 1.04) &(1.04, 1.04) &(1.00, 1.08)\\
0.2&       (1.03, 1.03) & (1.05, 1.06) &(1.05, 1.06) &(1.07, 1.12)\\
0.3&       (1.08, 1.09) & (1.11, 1.12) &(1.11, 1.12) &(1.17, 1.24)\\
0.4&       (1.09, 1.09) & {\bf (1.12, 1.13)} &(1.13, 1.14) &(1.22, 1.36)\\
0.5&       (1.10, 1.10) & {\bf (1.14, 1.15)} &(1.17, 1.18) &(1.30, 1.44)\\
0.6&       (1.12, 1.13) & {\bf (1.18, 1.19)} &(1.22, 1.23) &(1.46, 1.57)\\
0.7&       (1.13, 1.14) & {\bf (1.19, 1.20)} &(1.24, 1.25) &(1.60, 1.71)\\
0.8&       (1.12, 1.13) & {\bf (1.19, 1.20)} &(1.26, 1.27) &(1.64, 1.80)\\
0.9&       (1.15, 1.16) & {\bf (1.21, 1.22)} &(1.28, 1.29) &(1.80, 1.99)\\
1.0&       (1.14, 1.15) & {\bf (1.20, 1.20)} &(1.29, 1.30) &(2.02, 2.23)\\
1.1&       (1.17, 1.18) & {\bf (1.23, 1.24)} &(1.33, 1.34) &(2.14, 2.29)\\
1.2&       (1.18, 1.19) & {\bf (1.25, 1.26)} &(1.36, 1.37) &(2.33, 2.54)\\
1.3&       (1.14, 1.15) & {\bf (1.22, 1.23)} &(1.33, 1.34) &(2.37, 2.62)\\
1.4&       (1.14, 1.15) & {\bf (1.23, 1.24)} &(1.35, 1.36) &(2.50, 2.71)\\
1.5&       (1.17, 1.18) & {\bf (1.27, 1.28)} &(1.38, 1.39) &(2.73, 3.02)\\
1.6&       (1.17, 1.18) & {\bf (1.28, 1.29)} &(1.38, 1.39) &(2.77, 3.08)\\
1.7&       (1.16, 1.18) & {\bf (1.30, 1.31)} &(1.38, 1.39) &(2.86, 3.16)\\
1.8&       (1.17, 1.18) & {\bf (1.30, 1.31)} &(1.37, 1.38) &(2.78, 3.06)\\
1.9&       (1.21, 1.22) & {\bf (1.36, 1.37)} &(1.41, 1.42) &(2.98, 3.25)\\
\hline
\\&&{\normalsize $N = 200$}\\
0.2 &	(1.07, 1.07) & (1.08, 1.09) & (1.09, 1.09) & (1.13, 1.16)\\
0.4 &	(1.10, 1.11) & {\bf (1.14, 1.14)} & (1.18, 1.18) & (1.29, 1.39)\\
0.6 &	(1.13, 1.14) & {\bf (1.18, 1.19)} & (1.27, 1.27) & (1.49, 1.60)\\
0.8 &	(1.14, 1.15) & {\bf (1.20, 1.21)} & (1.34, 1.35) & (1.88, 2.03)\\
1.0 &	(1.16, 1.17) & {\bf (1.20, 1.20)} & (1.37, 1.38) & (2.23, 2.54)\\
1.2 &	(1.20, 1.21) & {\bf (1.25, 1.26)} & (1.47, 1.48) & (2.78, 2.94)\\	
1.4 &	(1.20, 1.21) & {\bf (1.27, 1.28)} & (1.50, 1.51) & (3.11, 3.35)\\
1.6 &	(1.21, 1.22) & {\bf (1.32, 1.32)} & (1.52, 1.53) & (3.38, 3.66)\\
1.8 &	(1.20, 1.20) & {\bf (1.34, 1.34)} & (1.48, 1.49) & (3.36, 3.68)\\
\hline
\end{tabular}
\label{table:mseVsAspectRatioCITableAppendix}
\end{table}

\newpage

\subsubsection{Performance as a function of noise level}

Fig.~\ref{fig:msesN100d90EpsVariousAbsolute} provides a wide view of the behavior of out-of-sample
mean squared error (MSE) as a function of noise level. Note the logarithmic scales, corresponding to
noise levels $10^{-2} \le \epsilon \le 10^2$ and MSE values $10^{-4} \le \text{mse} \le 10^4$. 
MSE of all of the regularization methods grows asymptotically like $\epsilon^2$, 
visually approximating the MSE noise lower bound.

\begin{figure}[h!]
    \centering
    \includegraphics[width=\figurewidth, trim = 0 2 4 10, clip=True]{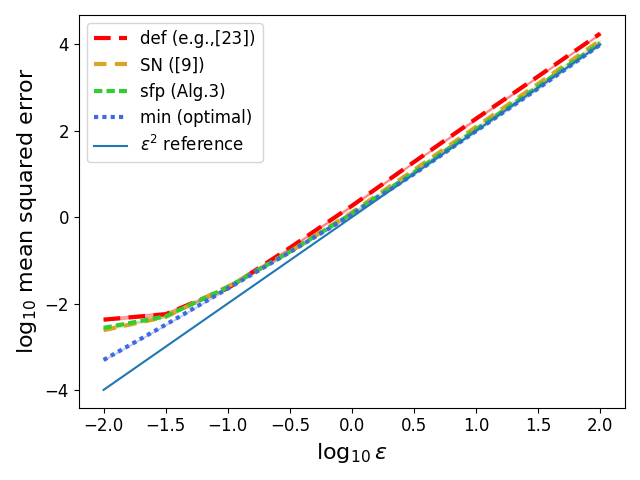}\\
    \caption{Median estimated out-of-sample MSE as a function of noise amplitude, $\epsilon$. 
    Spiked covariance model, $N=100$, $d=90$; results for bulk covariance model are similar,
    with slightly less vertical compression. 
    At this wide scale, all models approximate the reference $\text{MSE} = \epsilon^2$. 
    See Fig.~\ref{fig:msesN100d90EpsVariousRelative} for a closer view based on
    MSE relative to the optimal value.}
    \label{fig:msesN100d90EpsVariousAbsolute}
\end{figure}

Table~\ref{table:mseVsNoiseCINvariousTableAppendix} documents generalization at several 
noise levels, $\epsilon$, for sample sizes $N = 50, 100, 200$.
As discussed in section~\ref{subsection:mseVsNoise} of the main text, differences
between sample-based regularization methods are minor at smaller sample sizes ($N=50$ here), and increase with
sample size. Both Alg.~\ref{alg:sample-based-optimal-regularization} and signal-to-noise~\cite{DobribanWagnerHighDAsymptotics2018} outperform default regularization $\lambda_0 = 1$
at all noise levels at the aspect ratio $d/N = 0.9$ considered here.
Alg.~\ref{alg:sample-based-optimal-regularization} is the best sample-based performer at
higher noise levels $\epsilon > 1$ ($\log_{10} \epsilon > 0$), across sample sizes.

\begin{table}[ht]
\caption{Median $\log_{10}$ out-of-sample random-X MSE across noise levels (see Figs.~\ref{fig:msesN100d90EpsVariousRelative} and~\ref{fig:msesN100d90EpsVariousAbsolute}); best sample-based results in {\bf bold}. Spiked covariance, aspect ratio $\frac{d}{N} = 0.9$. BCa bootstrap $95\%$ confidence intervals based on $10^4$ bootstrap samples. Values aggregated over $100$ random-X repetitions for each of $100$ generative parameters $\theta$ as in Algorithm~\ref{alg:experimental-procedure-random-X}.}
\setlength{\tabcolsep}{4pt}
\centering
\begin{tabular}{ccccc}
{\normalsize $\log_{10} \epsilon$}		&{\normalsize $\lambda_{\text{min}}$}		&{\normalsize $\lambda_{\text{sfp}}$}	&{\normalsize $\lambda_{\text{SN}}$}	&{\normalsize $\lambda_0$}\\
\hline
\\&&{\normalsize $N = 50$}\\
-2.0&		(-3.33, -3.32) & (-2.35, -2.34) & {\bf (-2.37, -2.35)} & (-2.19, -2.11)\\
-1.5&       (-2.52, -2.51) & {\bf (-2.18, -2.17)} & (-2.17, -2.16) & (-2.08, -2.01)\\
-1.0&       (-1.68, -1.68) & {\bf (-1.62, -1.61)} & (-1.60, -1.59) & (-1.69, -1.62)\\
-0.5&       (-0.84, -0.83) & {\bf (-0.81, -0.81)} & (-0.81, -0.80) & (-0.79, -0.75)\\
0.0&       (0.03,  0.04) & {\bf (0.07,  0.07)} & (0.07,  0.08) & (0.20,  0.24)\\
0.5&       (1.00,  1.00) & {\bf (1.03,  1.03)} & (1.06,  1.06) & (1.21,  1.27)\\
1.0&       (1.98,  1.98) & {\bf (2.02,  2.03)} & (2.05,  2.06) & (2.18,  2.23)\\
1.5&       (2.97,  2.98) & {\bf (3.02,  3.02)} & (3.05,  3.05) & (3.17,  3.24)\\
2.0&       (3.98,  3.98) & {\bf (4.02,  4.03)} & (4.05,  4.06) & (4.17,  4.24)\\
\hline
\\&&{\normalsize $N = 100$}\\
-2.0 	& (-3.31, -3.3) & (-2.56, -2.55) & {\bf (-2.61, -2.61)} & (-2.42, -2.35)\\
-1.5	    & (-2.49, -2.48) & (-2.30, -2.29) & {\bf (-2.31, -2.30)} & (-2.25, -2.20)\\
-1.0		& (-1.65, -1.65) & {\bf (-1.62, -1.62)} & (-1.61, -1.60) & (-1.66, -1.62)\\
-0.5		& (-0.81, -0.80) & (-0.80, -0.79) & (-0.80, -0.79) & (-0.73, -0.68)\\
0.0		& ( 0.06,  0.06) & {\bf (0.08,  0.08)} & (0.10, 0.11) & (0.24, 0.29)\\
0.5		& ( 1.00,  1.00) & {\bf (1.03,  1.04)} & (1.09,  1.10) & (1.25, 1.27)\\
1.0		& ( 1.99,  1.99) & {\bf (2.03,  2.03)} & (2.09,  2.09) & (2.26, 2.30)\\
1.5		& ( 2.98,  2.98) & {\bf (3.02,  3.03)} & (3.09,  3.09) & (3.24, 3.28)\\
2.0		& ( 3.99,  4.00) & {\bf (4.03,  4.04)} & (4.09,  4.10) & (4.23, 4.27)\\
\hline
\\&&{\normalsize $N = 200$}\\
-2.0	  	& (-3.29, -3.29) & (-2.75, -2.75) & {\bf (-2.83, -2.82)} & (-2.62, -2.56)\\
-1.5		& (-2.47, -2.47) & (-2.37, -2.36) & {\bf (-2.38, -2.37)} & (-2.38, -2.34)\\
-1.0  	& (-1.64, -1.64) & {\bf (-1.63, -1.63)} & (-1.62, -1.62) & (-1.65, -1.62)\\
-0.5 	& (-0.80, -0.80) & (-0.79, -0.79) & (-0.79, -0.79) & (-0.68, -0.65)\\
0.0  	& ( 0.06,  0.06) & {\bf (0.08,  0.08)} & (0.14,  0.14) & (0.31,  0.35)\\
0.5    	& ( 1.01,  1.01) & {\bf (1.05,  1.05)} & (1.14,  1.14) & (1.31,  1.34)\\
1.0    	& ( 2.00,  2.00) & {\bf (2.04,  2.04)} & (2.13,  2.14) & (2.30,  2.34)\\
1.5    	& ( 3.00,  3.00) & {\bf (3.04,  3.04)} & (3.14,  3.14) & (3.32,  3.35)\\
2.0    	& ( 4.00,  4.00) & {\bf (4.04,  4.04)} & (4.13,  4.14) & (4.31,  4.34)\\
\hline
\end{tabular}
\label{table:mseVsNoiseCINvariousTableAppendix}
\end{table}

\subsubsection{Generative parameter estimates}

Table~\ref{table:genParameterEstimatesCITableAppendix} shows confidence intervals for the $L^2$ norm $\Vert \hat{\theta} \Vert$ of the estimated generative linear parameter, and for the estimated noise level $\hat{\epsilon}$.
Estimated noise levels are close to the true value $\epsilon = 1$, whereas estimated norms noticeably exceed the true $\Vert \theta \Vert = 1$ (even more so for spiked covariance).

\vspace{0.5cm}
\begin{table}[hb]
\caption{Generative parameter estimates vs.\ aspect ratio. Covariance as noted, noise $\epsilon=1$, sample size $N=100$. BCa bootstrap $95\%$ confidence intervals based on $10^4$ bootstrap samples. Values aggregated over $100$ random-X repetitions for each of $100$ generative unit vectors $\theta$ (Alg.~\ref{alg:experimental-procedure-random-X}).
Runs for aspect ratios $1.6, 1.8$ ended early due to matrix conditioning.}
\setlength{\tabcolsep}{6pt}
\centering
\begin{tabular}{ccc|cc}
&\hspace{0.9cm}{\normalsize Bulk}\hspace{-0.9cm}	&&\hspace{0.9cm}{\normalsize Spiked}\hspace{-0.9cm} &\\
{\normalsize $d/N$}		&{\normalsize $\Vert \hat{\theta} \Vert$}	&{\normalsize $\hat{\epsilon}$}	&{\normalsize $\Vert \hat{\theta} \Vert$}	&{\normalsize $\hat{\epsilon}$}	\\
\hline
0.1	 & (1.09, 1.10)	& (0.99, 1.00)		& (1.11, 1.12)	 & (1.00, 1.00)\\
0.2	 & (1.26, 1.27)	& (0.99, 0.99)		& (1.43, 1.44)	 & (1.00, 1.01)\\
0.3	 & (1.45, 1.46)	& (0.99, 1.00)		& (1.70, 1.71)	 & (1.01, 1.01)\\
0.4 	 & (1.64, 1.65)	& (0.99, 0.99)		& (1.92, 1.94)	 & (1.02, 1.02)\\
0.5	 & (1.85, 1.86)	& (0.98, 0.99)		& (2.12, 2.13)	 & (1.03, 1.03)\\
0.6	 & (2.06, 2.08)	& (0.98, 0.99)		& (2.30, 2.31)	 & (1.04, 1.05)\\
0.7	 & (2.28, 2.29)	& (0.98, 0.98)		& (2.45, 2.46)	 & (1.06, 1.07)\\
0.8	 & (2.48, 2.50)	& (0.99, 0.99)		& (2.60, 2.61)	 & (1.09, 1.10)\\
0.9	 & (2.65, 2.66)	& (1.01, 1.02)		& (2.72, 2.73)	 & (1.12, 1.13)\\
1.0	 & (2.72, 2.73)	& (1.07, 1.08)		& (2.82, 2.83)	 & (1.19, 1.19)\\
1.1	 & (2.69, 2.71)	& (1.04, 1.05)		& (2.90, 2.91)	 & (1.17, 1.18)\\
1.2	 & (2.58, 2.60)	& (1.02, 1.02)		& (2.96, 2.98)	 & (1.15, 1.15)\\
1.3	 & (2.42, 2.44)	& (1.00, 1.00)		& (3.00, 3.01)	 & (1.13, 1.13)\\
1.4	 & (2.27, 2.29)	& (0.98, 0.99)		& (3.00, 3.02)	 & (1.10, 1.11)\\
1.5	 & (2.13, 2.14)	& (0.97, 0.98)		& (2.98, 3.00)	 & (1.08, 1.08)\\
1.6	 & (1.99, 2.01)	& (0.96, 0.97)		& (2.94, 2.95)	 & (1.06, 1.06)\\
1.7	 & (1.88, 1.89)	& (0.95, 0.96)		& (2.86, 2.87)	 & (1.04, 1.04)\\
1.8	 & (1.79, 1.79)	& (0.94, 0.95)		& (2.77, 2.79)	 & (1.02, 1.02)\\
1.9	 & (1.70, 1.71)	& (0.94, 0.95)		& (2.69, 2.70)	 & (1.01, 1.01)\\
\hline
\end{tabular}
\label{table:genParameterEstimatesCITableAppendix}
\end{table}

\end{document}